\documentclass{article}
\usepackage{iclr2026_conference,times}
\iclrfinalcopy

\usepackage{amsmath,amssymb,amsfonts,amsthm,mathtools}

\usepackage[utf8]{inputenc}
\usepackage[T1]{fontenc}
\usepackage{hyperref}
\usepackage{url}
\usepackage{booktabs}
\usepackage{nicefrac}
\usepackage{microtype}
\usepackage{xcolor}
\definecolor{ForestGreen}{rgb}{0.13, 0.55, 0.13}
\usepackage{graphicx}
\usepackage{subcaption}
\usepackage{multirow}
\usepackage{colortbl}
\usepackage{float}
\usepackage{algorithm}
\usepackage{algorithmic}
\usepackage{natbib}

\usepackage[capitalize,noabbrev]{cleveref}

\usepackage{aliascnt}
\theoremstyle{plain}
\newtheorem{theorem}{Theorem}

\newaliascnt{proposition}{theorem}
\newtheorem{proposition}[proposition]{Proposition}
\aliascntresetthe{proposition}

\newaliascnt{lemma}{theorem}

\aliascntresetthe{lemma}

\newaliascnt{corollary}{theorem}

\aliascntresetthe{corollary}

\theoremstyle{definition}
\newaliascnt{definition}{theorem}

\aliascntresetthe{definition}

\newaliascnt{assumption}{theorem}

\aliascntresetthe{assumption}

\theoremstyle{remark}
\newaliascnt{remark}{theorem}
\newtheorem{remark}[remark]{Remark}
\aliascntresetthe{remark}

\crefname{proposition}{Proposition}{Propositions}
\Crefname{proposition}{Proposition}{Propositions}
\crefname{theorem}{Theorem}{Theorems}
\Crefname{theorem}{Theorem}{Theorems}
\crefname{lemma}{Lemma}{Lemmas}
\Crefname{lemma}{Lemma}{Lemmas}
\crefname{assumption}{Assumption}{Assumptions}
\Crefname{assumption}{Assumption}{Assumptions}

\newcommand{\R}{\mathbb{R}}

\DeclareMathOperator*{\argmin}{arg\,min}
\DeclareMathOperator*{\argmax}{arg\,max}
\DeclareMathOperator{\tr}{tr}
\DeclareMathOperator{\diag}{diag}
\DeclareMathOperator{\rank}{rank}
\DeclareMathOperator{\kerfn}{ker}

\newcommand{\Kinv}{\mathbf{K}^{-1}}            
\newcommand{\Hlik}{H_{\mathrm{lik}}}           
\newcommand{\Rmat}{\mathbf{R}}                 
\newcommand{\sexp}{s_{\mathrm{exp}}}            

\newcommand{\kappafn}{\kappa}                   
\newcommand{\etastar}{\eta^{*}}

\newcommand{\vf}{\mathbf{f}}

\newcommand{\mK}{\mathbf{K}}
\newcommand{\mH}{\mathbf{H}}
\newcommand{\mR}{\mathbf{R}}
\newcommand{\mC}{\mathbf{C}}
\newcommand{\mW}{\mathbf{W}}
\newcommand{\mI}{\mathbf{I}}
\newcommand{\mM}{\mathbf{M}}
\newcommand{\mB}{\mathbf{B}}

\newcommand{\normH}[1]{\left\|#1\right\|_{\bar{\mH}}}

\newcommand{\lmin}{\lambda_{\min}}

\newcommand{\Ncc}{N_{\mathrm{cc}}}

\newcommand{\probitscale}{\sigma}

\title{Adaptive KappaSharp: Condition-Number Shaping for Preferential Bayesian Optimization}

\author{
Ketong Shao$^{1}$ \quad Jialu Wang$^{2}$ \quad Xuekai Pei$^{3}$ \quad Ali Mesbah$^{1}$  \\[6pt] 
$^{1}$University of California, Berkeley \quad $^{2}$Independent \quad $^{3}$Wuhan University \\
\texttt{\{ketong\_shao, mesbah\}@berkeley.edu} \\
\texttt{luuyoyo1996@gmail.com} \quad \texttt{peixuekai@whu.edu.cn}
}

\begin{document}

\maketitle

\vspace{-5mm}
\begin{abstract}
Preferential Bayesian optimization (PBO) optimizes objectives accessible only through pairwise user comparisons. The standard approach fits a Gaussian process surrogate for observed pairwise comparisons (PairwiseGP) using the Laplace approximation and selects queries with the Expected Utility of Best Option (EUBO) acquisition function. EUBO queries new candidates at each step, producing pairs that share no candidates with previous queries. Each such pair forms an isolated component in the comparison graph, removing one degree of freedom from the likelihood Hessian and making it rank-deficient. This deficiency is structural and cannot be resolved by changing the surrogate modeling approach. Existing approaches to remedy this issue either waste query budget by forcing comparisons to stay connected, or apply uniform regularization that also perturbs directions already well-constrained by the observed comparisons. We propose KappaSharp that enables a diagonal correction to the Hessian to reduce its condition number, with larger corrections where the prior uncertainty is higher. The correction is only applied in the model fitting step, not query selection. An adaptive variant of KappaSharp is also presented that activates the correction only when the surrogate is confident about recent comparisons, avoiding unnecessary corrections when the problem is well-conditioned. On 11 benchmarks (5--20 dimensions), including a 16-dimensional controller tuning problem in plasma medicine, Adaptive KappaSharp outperforms the standard PBO baseline, with up to $+10.9\%$ ($p{=}0.003$).

\end{abstract}

\section{Introduction}
\label{sec:intro}

Bayesian optimization (BO)~\citep{brochu2010tutorial} is a popular approach for optimizing expensive black-box functions, but requires a scalar objective. Preferential BO (PBO) removes this requirement by learning from pairwise user comparisons~\citep{chu2005preference,gonzalez2017preferential}, with applications in exoskeleton gait optimization~\citep{tucker2019preference}, materials design~\citep{mikkola2020projective}, and model-based control design~\citep{shao2025coactive}. Standard PBO approaches fit a Gaussian process surrogate for observed pairwise comparisons (PairwiseGP) via the Laplace approximation~\citep{chu2005preference} and select queries using the Expected Utility of Best Option (EUBO) acquisition function~\citep{lin2022preferenceexploration,astudillo2023qeubo}, as implemented in BoTorch~\citep{balandat2020botorch}. Recent work has established regret bounds for PBO~\citep{xu2024principled}, but potential structural limitations of the standard querying strategy have not been investigated.

How queries are selected directly affects the Hessian that PairwiseGP uses to construct its posterior. In the Laplace approximation, the posterior covariance is the inverse of the Hessian of the negative log-posterior~\citep{rasmussen2006gaussian}. Hence, when the likelihood Hessian is rank-deficient, the full Hessian can become ill-conditioned and, consequently, the posterior can be poorly determined in certain directions. The problem is that observed pairwise comparisons are inherently limited: each comparison only informs the surrogate that one candidate's utility is likely higher than another's. That is, if $A$,$B$ are compared in one query and $C$,$D$ in another, the two utility gaps are learned independently. All comparisons can be represented as a \emph{comparison graph}, where nodes are candidates and edges are comparisons. Candidates connected through shared comparisons form \emph{connected components}; candidates in different components have no comparisons connecting them. We show that each disconnected component introduces one zero eigenvalue in the likelihood Hessian, and that these directions are not informed by the observed comparisons. With many disconnected components, the likelihood Hessian becomes severely rank-deficient and the full Hessian can become ill-conditioned.

This becomes a practical challenge because most acquisition functions used in PBO, including EUBO, select two candidates independently per query. When the decision variables are continuous, EUBO maximizes over the full input space, so the selected candidates almost never coincide with previously queried points~\citep{lin2022preferenceexploration}. This means each query can add a new disconnected component to the comparison graph, rather than connecting existing ones. Thus, the number of components grows with every query and the rank deficiency of likelihood Hessian worsens. Existing approaches either force the comparison graph to stay connected by constraining query selection~\citep{xu2024principled}, which limits exploration, or add uniform regularization to Hessian, which addresses the ill-conditioning at the expense of perturbing directions where the observed comparisons provide sufficient information. 

Here, we take a different approach: instead of altering the disconnected comparison graph, we look to modify the Hessian of PairwiseGP selectively. We propose KappaSharp, which improves conditioning of the Hessian by shaping its condition number ($\kappa$). Rather than regularizing all Hessian entries uniformly, a correction is added to its diagonal entries. The correction is larger where the prior uncertainty is high, and smaller where it is small. The correction directly affects the maximum a posteriori (MAP) estimation of the PairwiseGP posterior; not the query selection step. We additionally present an adaptive variant of KappaSharp that adaptively determines whether the correction must be applied based on PairwiseGP's confidence inferred from recent comparisons. \Cref{fig:hero} provides an overview of the comparison graph fragmentation problem and how KappaSharp addresses this challenge. 
KappaSharp does not modify the acquisition function or the querying strategy; thus, it can be applied as a drop-in addition to any PBO approach with minimal computational overhead.

\begin{figure}[h!]
    \centering
    \includegraphics[width=0.95\textwidth, trim=0 10 0 10, clip]{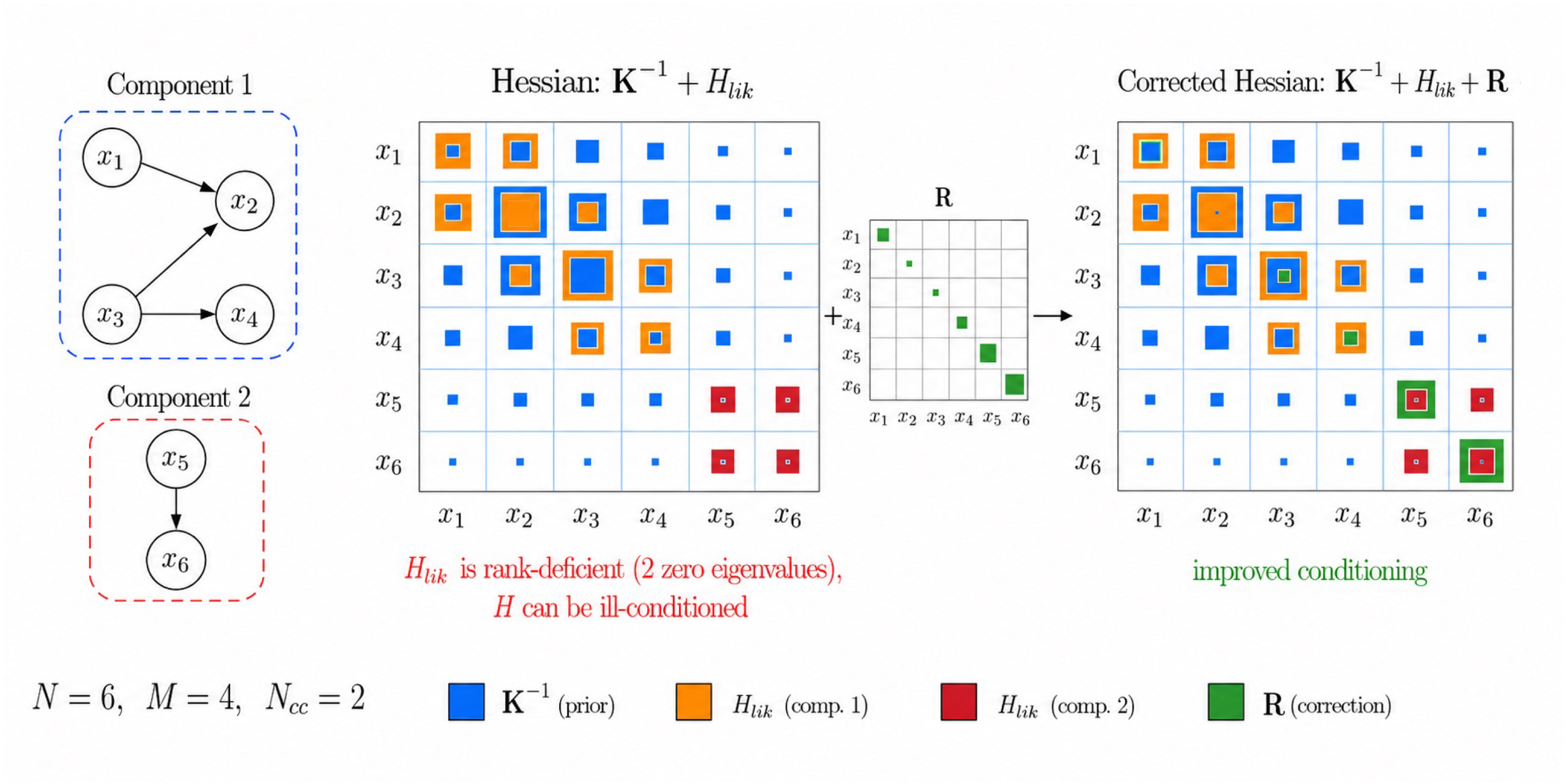}
    \vspace{-5mm}
    \caption{\textbf{Left:} An example comparison graph with $N{=}6$ candidates, $M{=}4$ pairwise comparisons, and $\Ncc{=}2$ disconnected components is shown alongside the corresponding Hessian $\mH = \Kinv + \Hlik$ and corrected Hessian $\mH + \Rmat$. Arrows indicate preference direction. \textbf{Middle:} The Hessian $\mH = \Kinv + \Hlik$. Blue cells show the prior precision $\Kinv$ (dense, all entries). Orange and red cells show the likelihood Hessian $\Hlik$, which contributes only within each connected component. The cross-component entries receive no contribution from the observed comparisons, so $\Hlik$ is rank-deficient and $\mH$ can be ill-conditioned. \textbf{Right:} The corrected Hessian $\mH + \Rmat$. Green diagonal cells show the KappaSharp correction $r_i = \eta^2/(\eta + (\Kinv)_{ii})$, which is larger where the prior precision is low. Square sizes reflect actual computed values for this example, illustrating that the correction is strongest where $\Kinv$ is weakest.}
    \label{fig:hero}
\end{figure}
\vspace{-5mm}
\paragraph{Contributions.} The main contributions of this work are as follows:
\begin{enumerate}
    \item[C1] We prove that standard PBO querying fragments the comparison graph, making the likelihood Hessian rank-deficient. This is a characteristic of the disconnected comparison graph and cannot be resolved by using alternative surrogate modelling approaches.

    \item[C2] We propose KappaSharp, a diagonal correction to the Hessian applied in MAP estimation of the PairwiseGP posterior, with larger corrections where the prior precision is low.

    \item[C3] We present an adaptive variant of KappaSharp that determines whether to apply the correction based on PairwiseGP's confidence inferred from recent comparisons.

    \item[C4] We evaluate the performance of KappaSharp on 11 benchmarks (5--20 dimensions) with 5 initializations and 60 seeds each, including a dose delivery problem in plasma medicine , in which PBO equipped with Adaptive KappaSharp improves over standard PBO without introducing losses.
\end{enumerate}

\section{Problem Setup}
\label{sec:problem}

We aim to find $x^* \in \argmax_{x \in \mathcal{X}} f(x)$ for a latent utility $f : \mathcal{X} \to \R$, $\mathcal{X} \subseteq \R^d$, that cannot be evaluated directly. At each iteration $t \in \{1, \ldots, T\}$, where $T$ is the total number of queries, we select a pair of candidates $(x_{t,a}, x_{t,b}) \in \mathcal{X} \times \mathcal{X}$ and observe only a binary preference. The user does not observe the utility values; instead, the preference is generated through noisy comparisons:
\begin{equation}
\label{eq:noisy_eval}
    y_{t,a} = f(x_{t,a}) + \epsilon_{t,a}, \quad y_{t,b} = f(x_{t,b}) + \epsilon_{t,b}, \quad \epsilon_{t,a}, \epsilon_{t,b} \overset{i.i.d.}{\sim} \mathcal{N}(0, \probitscale^2),
\end{equation}
with noise scale $\probitscale > 0$. The user reports $x_{t,a} \succ x_{t,b}$ if $y_{t,a} > y_{t,b}$. This yields the probit likelihood
\begin{equation}
\label{eq:probit}
    P(x_{t,a} \succ x_{t,b} \mid \vf) = \Phi\!\left(\frac{f(x_{t,a}) - f(x_{t,b})}{\probitscale\sqrt{2}}\right),
\end{equation}
where $\vf$ denotes the vector of latent utilities at all queried candidates and $\Phi$ is the standard normal cumulative distribution function (CDF). Each observation is recorded as an ordered pair $(x^+_t, x^-_t)$, where $x^+_t$ is the preferred candidate.

\section{Related Work}
\label{sec:related}

\paragraph{Preferential Bayesian Optimization.}
\citet{chu2005preference} introduced GP-based preference learning, and \citet{gonzalez2017preferential} formulated PBO as a sequential optimization approach. The EUBO acquisition function~\citep{lin2022preferenceexploration} and its batch variant qEUBO~\citep{astudillo2023qeubo} are currently the standard choice for query selection under the Laplace approximation. \citet{xu2024principled} provided the first regret bounds for PBO by enforcing a connected comparison graph, at the cost of reduced exploration. Alternative posterior approximations have also been proposed. \citet{benavoli2020skewgp} derived the exact skew-GP posterior; \citet{takeno2023practical} developed a practical implementation of SkewGP; and \citet{mikkola2020projective} proposed Gibbs sampling as an alternative. None of these developments addresses the rank deficiency problem identified in this work: when the comparison graph is disconnected, the observed comparisons provide no information linking different components (\Cref{prop:pbo_vulnerable}), regardless of the surrogate modelling approach. SkewGP also requires substantially more computation per posterior evaluation than the Laplace approximation.

\paragraph{Comparison Graphs and Spectral Structure.}
The connection between pairwise comparisons and graph Laplacians is investigated extensively in spectral ranking~\citep{jiang2011hodgerank}. \citet{shah2016estimation} showed that the Fisher information matrix for pairwise comparison models inherits the Laplacian structure. The fact that the probit Hessian takes a Laplacian form has been known since the work of \citet{chu2005preference}, but its impact on the PairwiseGP posterior under standard querying has not been analyzed. We show that EUBO with $q{=}2$ fragments the comparison graph into many disconnected components, causing the likelihood Hessian $\Hlik$ to lose rank, which will in turn make the full Hessian ill-conditioned (\Cref{sec:structural}).

\paragraph{Posterior Stabilization in Gaussian Processes.}
\citet{hartmann2019fisher} replaced the Hessian with the Fisher information matrix for GP classification. \citet{kundig2024iterative} introduced preconditioners for Vecchia-Laplace approximations. In continual learning, Elastic Weight Consolidation (EWC)~\citep{kirkpatrick2017overcoming} and online Laplace~\citep{immer2021improving} add diagonal regularization based on the curvature of the loss. These methods do not account for the graph-Laplacian rank structure that arises in pairwise comparison models. When $\Hlik$ is rank-deficient, the Hessian does not receive contribution from the observed comparisons in certain directions. Thus, the eigenvalues of the Hessian are determined solely by the prior precision, which can be much smaller than in directions where $\Hlik$ contributes. The proposed correction scheme addresses this imbalance by shaping the condition number of the Hessian.

\section{Structural Analysis}
\label{sec:structural}

The iterative PBO procedure starts from $N_0$ initial candidate points $x_1, \ldots, x_{N_0} \in \mathcal{X}$ with $M_0$ comparisons among them, forming $N_{\mathrm{cc},0}$ connected components in the comparison graph. As discussed in \Cref{sec:intro}, under EUBO with batch size $q{=}2$ (two candidate points per query), each subsequent query selects two new candidate points independent of all previous ones. After $t{-}1$ such queries, we have $N_{t-1} = N_0 + 2(t{-}1)$ candidates and $M_{t-1} = M_0 + (t{-}1)$ comparisons. With a slight abuse of notation, we write $N = N_{t-1}$ and $M = M_{t-1}$ for the number of candidates and comparisons at iteration $t$, respectively.

A GP prior $\vf \sim \mathcal{N}(\mathbf{0}, \mK)$ is placed over the latent utility values $\vf = (f(x_1), \ldots, f(x_N))^\top$, where $\mK \in \R^{N \times N}$ with $K_{ij} = k(x_i, x_j)$ for a kernel function $k$. Accordingly, the $m$-th comparison records that $x^+_m$ was preferred over $x^-_m$. We encode all $M$ comparisons in an incidence matrix $\mC \in \R^{M \times N}$, where row $m$ has $+1$ in the column corresponding to $x^+_m$, $-1$ in the column corresponding to $x^-_m$, and zero elsewhere. Let $\mathcal{D} = \{(x^+_m, x^-_m)\}_{m=1}^M$ denote the set of observed pairwise comparisons, where each pair records the preferred and non-preferred candidate. The Laplace approximation gives the posterior $p(\vf \mid \mathcal{D}) \approx \mathcal{N}(\hat{\vf}, \mH^{-1})$, where $\hat{\vf} = \argmin_{\vf}\, -\log p(\mathcal{D} \mid \vf) + \tfrac{1}{2}\vf^\top \mK^{-1} \vf$ is the MAP estimate of $\vf$ and
\begin{equation}
\label{eq:precision}
    \mH = \mK^{-1} + \underbrace{\mC^\top \mW \mC}_{\Hlik}
\end{equation}
is the posterior precision matrix, with $\mW = \diag(\omega_1, \ldots, \omega_M)$ a diagonal matrix of positive weights that depend on $\hat{\vf}$, as defined in \Cref{prop:graph_laplacian}. The likelihood Hessian $\Hlik$ has a specific structure.

\begin{proposition}[$\Hlik$ is a weighted graph Laplacian]
\label{prop:graph_laplacian}
For every $\vf \in \R^N$ with finite comparison scores, the data-dependent component of the Laplace precision, $\Hlik(\vf) =
-\nabla^2_{\vf} \log p(\mathcal{D} \mid \vf)$, satisfies
\begin{equation}
\label{eq:hlik}
    \Hlik(\vf) = \mC^\top \mW(\vf)\, \mC,
\end{equation}
where $\mW = \diag(\omega_1, \ldots, \omega_M)$ with $\omega_m = (2\probitscale^2)^{-1}\psi(z_m)\bigl(\psi(z_m) + z_m\bigr) > 0$, $z_m = \mC_m^\top \vf / (\probitscale\sqrt{2})$, $\phi$ is the standard normal probability density function (PDF), $\Phi$ is the CDF, and $\psi(z) = \phi(z)/\Phi(z)$ is the inverse Mills ratio.
The rank of $\Hlik$ is
\begin{equation}
    \rank(\Hlik) = N - \Ncc,
\end{equation}
where $\Ncc$ is the number of connected components of the comparison graph.
\end{proposition}

Within each connected component, $\Hlik$ constrains relative utility differences. That is, the likelihood cannot determine how utilities in one component relate to utilities in another. Proofs for all propositions and theorems in this section and \Cref{sec:method} are given in \Cref{app:proofs}.

\begin{proposition}[Likelihood-null directions]
\label{prop:pbo_vulnerable}
Let $\kerfn(\mC) = \{\vf \in \R^N : \mC \vf = \mathbf{0}\}$ be the null space of $\mC$ and $\mathrm{range}(\mC^\top)$ be its orthogonal complement. Let $\mathbf{Q}_0 \in \R^{N \times \Ncc}$ and $\mathbf{Q}_1 \in \R^{N \times (N-\Ncc)}$ be orthonormal bases for $\kerfn(\mC)$ and $\mathrm{range}(\mC^\top)$, respectively. Decompose the latent utility vector as $\vf = \mathbf{Q}_0 \xi + \mathbf{Q}_1 \zeta$, where $\xi \in \R^{\Ncc}$ are the null-space coordinates and $\zeta \in \R^{N-\Ncc}$ are the range-space coordinates. Then:
\begin{enumerate}
\item[(a)] The likelihood does not depend on $\xi$: $p(\mathcal{D} \mid \xi, \zeta) = p(\mathcal{D} \mid \zeta)$.
\item[(b)] Under the Laplace approximation, the posterior over $\xi$ given $\zeta$ equals the prior: $p_L(\xi \mid \zeta, \mathcal{D}) = p_0(\xi \mid \zeta)$.
\end{enumerate}
\end{proposition}

The null-space coordinates $\xi$ represent a constant shift in the utility of all candidates within the same connected component, i.e., a group of candidate points $x$ that are linked through a chain of pairwise comparisons. Since comparisons only measure utility \emph{differences} within a pair, such shifts do not change any comparison outcome, and the posterior over $\xi$ is determined entirely by the GP prior.

\paragraph{Comparison graph fragmentation under EUBO.}
Given the posterior $\mathcal{N}(\hat{\vf}, \mH^{-1})$, the EUBO acquisition function with $q{=}2$ selects the pair $(x_a, x_b)$ that maximizes the expected utility of the best option in the pair. For a pair with posterior mean difference $\delta = \hat{f}(x_a) - \hat{f}(x_b)$ and predictive standard deviation $s = \sqrt{\mathrm{Var}[f(x_a) - f(x_b) \mid \mathcal{D}]}$, the EUBO score is then
\begin{equation}
\label{eq:eubo}
    \mathrm{EUBO}(x_a, x_b) = \delta\,\Phi(\delta/s) + s\,\phi(\delta/s),
\end{equation}
In continuous decision domains, these candidates are almost always new points~\citep{lin2022preferenceexploration}. Thus, after $T$ queries from $N_0$ initial points with $N_{\mathrm{cc},0}$ connected components, there are $N_0 + 2T$ candidates in up to $N_{\mathrm{cc},0} + T$ components.

We now quantify how much the PairwiseGP surrogate can learn from the observed pairwise comparisons using the \emph{effective degrees of freedom} $d_{\mathrm{eff}} \in [0, N)$. Let $\tilde{\mM} = \mK^{1/2}\Hlik\mK^{1/2}$, which captures how much information the observed comparisons add beyond the prior in each direction of the latent utility space. Each eigenvalue $\mu_i$ of $\tilde{\mM}$ corresponds to one such direction, and $\mu_i/(1+\mu_i)$ is close to 1 when the comparisons dominate and close to 0 when the prior dominates. $d_{\mathrm{eff}} = \sum_i \mu_i/(1+\mu_i)$ counts how many directions are effectively informed by the observed comparisons.

\begin{proposition}[Effective information dimension]
\label{prop:info_dim}
Given $\tilde{\mM}$ and $d_{\mathrm{eff}}$ as defined above:
\begin{enumerate}
\item[(a)] $d_{\mathrm{eff}} \leq \rank(\Hlik) = N - \Ncc$, with equality only in the limit $\mu_i \to \infty$.
\item[(b)] Under EUBO $q{=}2$, assume each iteration queries a comparison between two previously unseen points and the initial $N_0$ points form $N_{\mathrm{cc},0}$ connected components. Then, $N = 2T + N_0$, $\Ncc = N_{\mathrm{cc},0} + T$, and $d_{\mathrm{eff}}/N \leq (N_0 + T - N_{\mathrm{cc},0})/(N_0 + 2T)$, which approaches $1/2$ as $T \to \infty$.
\item[(c)] Under connected querying ($\Ncc{=}1$): $d_{\mathrm{eff}} \leq N{-}1$.
\item[(d)] With any positive-definite augmentation $\Rmat \succ 0$ satisfying $\mK^{1/2}\Rmat\mK^{1/2} \succeq \beta\mI$ for some scalar $\beta > 0$, define $d_{\mathrm{eff}}^{\mathrm{aug}}$ as the effective degrees of freedom computed from $\mH + \Rmat$ (i.e., replacing $\tilde{\mM}$ by $\mK^{1/2}(\Hlik + \Rmat)\mK^{1/2}$). Then:
\begin{equation}
\label{eq:deff_recovery}
d_{\mathrm{eff}}^{\mathrm{aug}} \geq d_{\mathrm{eff}} + \frac{\Ncc\beta}{1+\beta}.
\end{equation}
\end{enumerate}
\end{proposition}

In \Cref{prop:info_dim}, part~(a) indicates that each disconnected component removes one effective degree of freedom. Part~(b) indicates that under standard EUBO, at most half the degrees of freedom are informed by the observed pairwise comparisons. Part~(c) indicates that connected querying removes this ceiling. Part~(d) indicates that adding a positive-definite matrix to the Hessian can increase $d_{\mathrm{eff}}$ algebraically, but this does not represent real information from comparisons (see \Cref{sec:method}).

\section{Condition-Number Shaping of Hessian}
\label{sec:method}

In standard PBO, the posterior of PairwiseGP is estimated via MAP estimation in every iteration. Let $J_0(\vf) = -\log p(\mathcal{D} \mid \vf) + \tfrac{1}{2}\vf^\top \Kinv \vf$ denote the negative log-posterior up to a constant, and let $\hat{\vf}_0 = \argmin_{\vf} J_0(\vf)$ be the standard MAP estimate. As shown in \Cref{sec:structural}, the likelihood Hessian $\Hlik$ is rank-deficient when the comparison graph is disconnected, which can make the full Hessian $\mH = \Kinv + \Hlik(\hat{\vf}_0)$ ill-conditioned. We propose to replace the standard MAP estimation with the augmented problem
\begin{equation}
\label{eq:augmented_map}
    \hat{\vf}_\eta = \argmin_{\vf}\, J_0(\vf) + \tfrac{1}{2}\vf^\top \Rmat(\eta)\, \vf,
\end{equation}
where $\Rmat(\eta) = \diag(r_1(\eta), \ldots, r_N(\eta))$ is a diagonal correction matrix that adds a penalty to each diagonal entry of the Hessian. The scalar $\eta > 0$ controls the overall correction strength. After solving for $\hat{\vf}_\eta$, we evaluate the Hessian of $J_0$ (not the augmented objective) at $\hat{\vf}_\eta$, denoted as $\mH_\eta = \Kinv + \Hlik(\hat{\vf}_\eta)$, and use its inverse as the Laplace covariance: $\Sigma_{\mathrm{MO}} = \mH_\eta^{-1}$. That is, $\Rmat$ shifts the MAP from $\hat{\vf}_0$ to $\hat{\vf}_\eta$ but does not enter the covariance. We describe the design of $\Rmat$ next and justify the deployment of MAP-only in the following.

\subsection{KappaSharp Correction}
\label{sec:filter}

Each diagonal entry $r_i(\eta)$ depends on the prior precision $(\Kinv)_{ii}$ at the $i$-th candidate point.

\begin{proposition}[Correction properties]
\label{prop:filter_properties}
Define $a_i = (\Kinv)_{ii} > 0$, which denotes the prior precision at the $i$-th candidate point, and let
\begin{equation}
\label{eq:filter}
    r_i(\eta) = \frac{\eta^2}{\eta + a_i}.
\end{equation}
Then, for $\eta > 0$:
\begin{enumerate}
    \item[(a)] $\partial r_i / \partial a_i < 0$: larger prior precision $\to$ less correction;
    \item[(b)] $0 < r_i < \eta$: the correction is bounded;
    \item[(c)] as $a_i \to 0$: $r_i \to \eta$ (maximum correction where the prior is weak);
    \item[(d)] as $a_i \to \infty$: $r_i \to 0$ (no correction where the prior already dominates).
\end{enumerate}
\end{proposition}

The correction uses only the diagonal entries $(\Kinv)_{ii}$, not the likelihood Hessian $\Hlik$ directly. Adding $\Rmat$ to the Hessian is equivalent to strengthening the prior precision from $\Kinv$ to $\Kinv + \Rmat$. The ill-conditioning of $\mH = \Kinv + \Hlik$ is worst in directions where $\Hlik$ contributes nothing and $\Kinv$ is weak. Under EUBO querying in continuous domains, these two conditions tend to occur simultaneously: newly queried points that form isolated components in the comparison graph tend to have low kernel correlation with existing points, so both their likelihood contribution and their prior uncertainty are high. By applying larger corrections where $(\Kinv)_{ii}$ is small, $\Rmat$ boosts the weakest eigenvalues of $\mH$, improving its conditioning without requiring direct access to the structure of $\Hlik$. The overall correction strength $\etastar$ is calibrated using a condition-number target on the full Hessian $\mH + \Rmat$, which accounts for $\Hlik$ indirectly.  The condition number $\kappafn(\mH) = \lambda_{\max}(\mH)/\lambda_{\min}(\mH)$ measures the ratio between the largest and smallest eigenvalues of $\mH$. A large $\kappafn(\mH)$ indicates that some directions are much better determined than others. We choose $\etastar$ as the smallest $\eta$ that reduces $\kappafn(\mH + \Rmat(\eta))$ to a target level, which we call $\kappafn$-calibration.
\label{sec:calibration}

\begin{proposition}[$\kappafn$-calibration is well-posed]
\label{prop:calibration}
Let $\mH = \mK^{-1} + \Hlik(\hat{\vf}_0)$ be the Hessian at $\hat{\vf}_0$ and $\alpha \in (0,1)$.
The correction $\Rmat(\eta)$ is continuous, positive semi-definite, and monotonically increasing in $\eta$ (i.e., $\eta' > \eta \Rightarrow \Rmat(\eta') \succeq \Rmat(\eta)$); $\kappafn(\mH + \Rmat(\eta)) \to 1$ as $\eta \to \infty$; and therefore
\begin{equation}
\label{eq:calibration}
    \etastar = \underset{\eta}{\min}\;\left\{\eta \geq 0 : \kappafn(\mH + \Rmat(\eta)) \leq \kappafn(\mH)^{1-\alpha}\right\}
\end{equation}
exists and is finite.
\end{proposition}

We set $\alpha = 0.1$ throughout, targeting $\kappafn(\mH)^{0.9}$. A sensitivity analysis over $\alpha$ is shown in \Cref{app:alpha_sensitivity}.

\paragraph{Bounded MAP shift.}
If the correction shifts the MAP too far from the original estimate, the resulting posterior may no longer reflect the observed pairwise comparisons. The following proposition bounds the shift $\Delta = \hat{\vf}_\eta - \hat{\vf}_0$.

\begin{proposition}[Bounded MAP shift]
\label{prop:bias_bound}
The shift $\Delta = \hat{\vf}_\eta - \hat{\vf}_0$ satisfies $\normH{\Delta} \leq \frac{\mu_{\max}}{1 + \mu_{\max}} \normH{\hat{\vf}_0}$, where $\|\cdot\|_{\bar{\mH}}$ is the norm induced by the path-averaged Hessian $\bar{\mH}$ of $J_0$, and $\mu_{\max}$ is the largest eigenvalue of $\bar{\mH}^{-1/2}\Rmat\bar{\mH}^{-1/2}$ (the path-averaged Hessian and its norm are defined in the proof in \Cref{app:proofs}).
\end{proposition}

Note that shifting the MAP changes the weights $\mW$ but not the rank of $\Hlik$: since $\mW(\vf) \succ 0$ for all finite $\vf$, $\rank(\Hlik(\vf)) = N - \Ncc$ regardless of $\vf$ (see \Cref{prop:rank_invariance_statement} in \Cref{app:proofs})).

\paragraph{Why $\Rmat$ should not enter the covariance.}
Recall from \Cref{prop:pbo_vulnerable} that the observed pairwise comparisons provide no information in the null-space directions $\kerfn(\mC)$; thus, the posterior covariance in those directions should equal the prior. If we were to include $\Rmat$ in the covariance, the added $\Rmat$ would shrink the posterior covariance in the null-space directions to values smaller than the prior covariance. PairwiseGP would then underestimate its uncertainty in those directions, and EUBO would stop exploring them. By excluding $\Rmat$ from the covariance, MAP-only preserves the correct uncertainty in the null space.

\begin{theorem}[Pseudo-confidence in null directions]
\label{thm:pseudo_confidence}
Using the decomposition $\vf = \mathbf{Q}_0 \xi + \mathbf{Q}_1 \zeta$ from \Cref{prop:pbo_vulnerable}, let $\mathbf{A} = \mathbf{Q}_0^\top \Kinv \mathbf{Q}_0$ be the prior precision on the null-space coordinates $\xi$. Full augmentation $(\mH_\eta{+}\Rmat)^{-1}$ shrinks the conditional covariance of $\xi$, given $\zeta$, to $(\mathbf{A}{+}\mathbf{Q}_0^\top \Rmat \mathbf{Q}_0)^{-1} \prec \mathbf{A}^{-1}$ since $\mathbf{Q}_0^\top \Rmat \mathbf{Q}_0 \succ 0$. MAP-only preserves $\mathbf{A}^{-1}$, matching the exact posterior (\Cref{prop:pbo_vulnerable}).
\end{theorem}

Full augmentation also reduces EUBO scores for all candidate pairs, which discourages exploration.

\begin{proposition}[Full augmentation contracts EUBO scores]
\label{prop:conservative}
Let $\Sigma_{\mathrm{aug}} := (\mH_\eta + \Rmat)^{-1}$. For any candidate pair $(x_a, x_b)$, let $s^2_\Sigma(a,b)$ denote the predictive variance of the utility difference $f(x_a) - f(x_b)$ under covariance $\Sigma$. Since $\Sigma_{\mathrm{aug}} \preceq \Sigma_{\mathrm{MO}}$, we have $s_{\mathrm{aug}}^2(a,b) \leq s_{\mathrm{MO}}^2(a,b)$ for every pair. Since $q{=}2$ EUBO is strictly increasing in the predictive standard deviation at fixed mean, the EUBO score under full augmentation is lower than under MAP-only for every candidate pair:
\begin{equation}
\label{eq:conservative}
\mathrm{EUBO}_{\mathrm{aug}}(a,b) \leq \mathrm{EUBO}_{\mathrm{MO}}(a,b) \quad \text{for every pair } (x_a, x_b).
\end{equation}
\end{proposition}

\paragraph{The cost of MAP-only is bounded.}
Excluding $\Rmat$ from the covariance avoids both issues. The acquisition gap from between the MAP-only and full-augmentation strategies is bounded.

\begin{theorem}[Bounded acquisition gap]
\label{thm:one_step_regret}
The difference in EUBO between the MAP-only maximizer and the full-augmentation maximizer is at most $O(\lambda_1)$, where $\lambda_1$ is the largest eigenvalue of $\mH_\eta^{-1/2}\Rmat\,\mH_\eta^{-1/2}$, measuring how large the correction $\Rmat$ is relative to the Hessian $\mH_\eta$. The precise bound is given in \Cref{app:proofs} and is finite for any continuous kernel on a compact domain.
\end{theorem}

\subsection{Adaptive Activation}
\label{sec:adaptive}

Applying KappaSharp at every iteration, referred to as \emph{Static KS} in our experiments, is beneficial when the rank deficiency of the likelihood Hessian is the primary bottleneck limiting the surrogate. However, it can be detrimental when the prior already provides sufficient conditioning. We observe that PairwiseGP's confidence on recently queried pairs separates these two regimes: when the rank deficiency of the likelihood Hessian dominates, the surrogate quickly becomes confident about individual comparisons; when the prior is already well-conditioned, the surrogate remains uncertain. The adaptive variant of KappaSharp uses this separation to decide when to activate the correction.

\paragraph{Decisiveness score.}
After fitting the surrogate on $\mathcal{D}_{t-1}$, we measure how confidently it distinguishes a recently compared pair.
Let $\mu_{\mathrm{diff}} = |\hat{f}(x^+_{t-1}) - \hat{f}(x^-_{t-1})|$ be the absolute posterior mean difference and $\sigma_{\mathrm{diff}}^2$ the posterior variance of $f(x^+_{t-1}) - f(x^-_{t-1})$.
The surrogate's predicted probability that the comparison outcome matches its ranking is $p_t = \Phi\!\left(\mu_{\mathrm{diff}} / \sqrt{\sigma_{\mathrm{diff}}^2 + 2\probitscale^2}\right) \in [1/2, 1)$, where $2\probitscale^2$ is the comparison noise variance from the probit model (\Cref{eq:noisy_eval}).
A high value of $p_t$ indicates that the surrogate predicts a clear utility difference between the two candidates, whereas $p_t \approx 0.5$ indicates that the comparison is ambiguous. We define the \emph{decisiveness score} as
\begin{equation}
\label{eq:sexp}
    s_{\mathrm{exp},t} = 1 - \frac{\mathcal{H}(p_t)}{\log 2},
\end{equation}
where $\mathcal{H}(p) = -p\log p - (1{-}p)\log(1{-}p)$ is the binary entropy. The score ranges from $0$ (coin flip) to $1$ (decisive).

\paragraph{Activation rule.}
We track an exponential moving average to smooth the step-to-step variance of $s_{\mathrm{exp},t}$:
\begin{equation}
\label{eq:ema}
    \bar{s}_t = (1 - \alpha_{\mathrm{ema}})\,\bar{s}_{t-1} + \alpha_{\mathrm{ema}} s_{\mathrm{exp},t}, \quad \alpha_{\mathrm{ema}} = 0.2,
\end{equation}
initialized at $\bar{s}_0 = 0$.
The correction activates at step $t$ when $\bar{s}_t \geq \tau$, with threshold $\tau = 0.30$. In each step, $\tau = 0.30$ corresponds to $p_t \geq 0.81$ (\Cref{prop:margin_safe} in \Cref{app:threshold_interp}). The activation rule uses the EMA $\bar{s}_t$, not the per-step score. Thus, $\bar{s}_t \geq \tau$ indicates sufficient average decisiveness over recent comparisons. \Cref{alg:adaptive_ks} summarizes the Adaptive KappaSharp approach. When the correction is active, the additional computation consists of reading the diagonal of $\Kinv$ to form $\Rmat$, searching over 60 log-spaced values of $\eta$ to solve the $\kappafn$-calibration problem, and running one extra Newton solve for the augmented MAP. The better-conditioned Hessian often reduces the number of Newton iterations needed, partially or fully offsetting this cost (\Cref{app:wallclock}). When the correction is inactive, the only overhead per step is evaluating $\sexp$ from a single posterior prediction.

\begin{algorithm}[t]
\caption{Adaptive KappaSharp for Preferential BO}
\label{alg:adaptive_ks}
\begin{algorithmic}[1]
\REQUIRE Initial data $\mathcal{D}_0$, GP prior $\mK$, threshold $\tau = 0.30$, calibration $\alpha = 0.1$, EMA rate $\alpha_{\mathrm{ema}} = 0.2$, minimum step $t_{\min} = 8$
\STATE Initialize $\bar{s}_0 \gets 0$
\FOR{$t = 1, 2, \ldots, T$}
    \STATE Fit PairwiseGP on $\mathcal{D}_{t-1}$: compute MAP $\hat{\vf}_0$, Hessian $\mH = \Kinv + \Hlik$
    \STATE Compute $s_{\mathrm{exp},t}$ from PairwiseGP's prediction on most recent comparison \hfill \COMMENT{Eq.~\eqref{eq:sexp}}
    \STATE Update $\bar{s}_t \gets (1 - \alpha_{\mathrm{ema}})\,\bar{s}_{t-1} + \alpha_{\mathrm{ema}} s_{\mathrm{exp},t}$ \hfill \COMMENT{Eq.~\eqref{eq:ema}}
    \IF{$t \geq t_{\min}$ \AND $\bar{s}_t \geq \tau$}
        \STATE Compute $a_i = (\Kinv)_{ii}$ and $\Rmat(\eta) = \diag(\eta^2/(\eta + a_i))$ \hfill \COMMENT{Eq.~\eqref{eq:filter}}
        \STATE Solve $\etastar = \underset{\eta}{\min}\;\{\eta \geq 0 : \kappafn(\mH + \Rmat(\eta)) \leq \kappafn(\mH)^{1-\alpha}\}$ \hfill \COMMENT{Eq.~\eqref{eq:calibration}}
        \STATE Re-solve MAP: $\hat{\vf}_\eta \gets \argmin J_0(\vf) + \tfrac{1}{2}\vf^\top \Rmat(\etastar)\vf$ \hfill \COMMENT{Eq.~\eqref{eq:augmented_map}}
        \STATE Set approximate posterior: $\mathcal{N}(\hat{\vf}_\eta,\; \mH^{-1}|_{\hat{\vf}_\eta})$ \hfill \COMMENT{Unaugmented covariance at shifted MAP}
    \ELSE
        \STATE Use standard posterior: $\mathcal{N}(\hat{\vf}_0,\; \mH^{-1}|_{\hat{\vf}_0})$
    \ENDIF
    \STATE Select next query $(x_a, x_b)$ via EUBO on the current posterior
    \STATE Observe comparison outcome; update $\mathcal{D}_t \gets \mathcal{D}_{t-1} \cup \{(x^+_t, x^-_t)\}$
\ENDFOR
\end{algorithmic}
\end{algorithm}
\vspace{-4mm}

\section{Experiments}
\label{sec:experiments}

We benchmark Adaptive KappaSharp against the standard PairwiseGP baseline and several alternatives on 11 problems spanning 5 to 20 dimensions. The baselines include Static KS (KappaSharp applied at every iteration without the activation rule), Fixed20 (correction activated only after step 20, without the decisiveness criterion), Connected Querying (one candidate per query is fixed to the previous query point $x'_t = x_{t-1}$, following the querying strategy of \citet{xu2024principled}; EUBO is retained as the acquisition function for fair comparison), and the unmodified PairwiseGP-Laplace + EUBO baseline (B). We report ablations and diagnostic comparisons with non-Laplace inference.

\subsection{Setup}
\label{sec:setup}

\paragraph{Benchmarks.}
The 11 benchmarks span 5 to 20 dimensions: DTLZ2-8D, DTLZ2-16D, DTLZ2-20D (multi-objective PBO test functions), Plasma-16D (model-based control design in plasma medicine~\citep{shao2025coactive}), Ackley-8D, Hartmann-6D, Levy-10D, Levy-20D, Vehicle-5D (single-objective), CarCab-7D (vehicle cabin design), and RobotPush-14D (Box2D physics simulation).

\paragraph{Protocol.}
All methods use PairwiseGP with the Laplace approximation and EUBO acquisition~\citep{lin2022preferenceexploration}, implemented in BoTorch~\citep{balandat2020botorch}. Each run consists of $T = 50$ sequential pairwise queries, with comparison noise $\probitscale = 0.1$. We test five initialization conditions: \texttt{pool\_n6\_k15} (6 candidates, all 15 pairwise comparisons; primary), \texttt{pool\_n8\_k15} (8 candidates, 15 comparisons), and three pure-matching conditions \texttt{pure\_k3}, \texttt{pure\_k5}, \texttt{pure\_k15} ($k$ non-overlapping pairs among $2k$ candidates). Each condition uses 60 seeds with paired same-node execution and common random numbers (CRN). The total number of runs in the main grid is $16{,}500$ (11 benchmarks $\times$ 5 inits $\times$ 5 methods $\times$ 60 seeds).

\paragraph{Evaluation.}
We report percentage improvement in final best utility over baseline B. Significance is assessed via paired $t$-test ($p < 0.05$, uncorrected). A benchmark-method pair is classified as win (W), loss (L), or neutral (N) based on the sign and significance of the improvement.

\subsection{Baselines and ablations}
\label{sec:diagnosis}

\begin{table*}[!h]
\centering
\caption{Benchmark results (pool\_n6\_k15 init, 60 seeds). B: PairwiseGP-Laplace + EUBO baseline. Static KS: KappaSharp without activation rule. Fixed20: correction on at step $t{\geq}20$. Conn.\ Query: connected querying strategy of~\citet{xu2024principled}, with EUBO retained for fair comparison. Adaptive KS: proposed method with confidence-based activation ($\tau{=}0.30$). Gain (\%) relative to B. \textbf{Bold}: best mean. $^{*}$: $p{<}0.05$, $^{**}$: $p{<}0.01$ (paired $t$-test).}
\label{tab:main}
\resizebox{\textwidth}{!}{%
\begin{tabular}{l c | c | cc | cc | cc | cc}
\toprule
 & & \multicolumn{1}{c|}{B} & \multicolumn{2}{c|}{Static KS} & \multicolumn{2}{c|}{Fixed20} & \multicolumn{2}{c|}{Conn.\ Query} & \multicolumn{2}{c}{Adaptive KS} \\
Benchmark & $d$ & mean$\pm$std & mean$\pm$std & Gain & mean$\pm$std & Gain & mean$\pm$std & Gain & mean$\pm$std & Gain \\
\midrule
  DTLZ2-16D & 16 & -1.031$\pm$0.341 & \textbf{-0.778$\pm$0.281} & \textcolor{ForestGreen}{+24.5\%$^{**}$} & -0.783$\pm$0.282 & \textcolor{ForestGreen}{+24.0\%$^{**}$} & -1.016$\pm$0.327 & +1.5\% & -0.875$\pm$0.265 & \textcolor{ForestGreen}{+15.1\%$^{**}$} \\
  DTLZ2-8D & 8 & -0.184$\pm$0.134 & \textbf{-0.142$\pm$0.068} & \textcolor{ForestGreen}{+23.2\%$^{*}$} & -0.167$\pm$0.114 & +9.6\% & -0.357$\pm$0.189 & \textcolor{red}{-93.9\%$^{**}$} & -0.163$\pm$0.101 & +11.7\% \\
  DTLZ2-20D & 20 & -1.439$\pm$0.390 & \textbf{-1.183$\pm$0.379} & \textcolor{ForestGreen}{+17.8\%$^{**}$} & -1.256$\pm$0.453 & \textcolor{ForestGreen}{+12.7\%$^{**}$} & -1.350$\pm$0.342 & +6.2\% & -1.268$\pm$0.444 & \textcolor{ForestGreen}{+11.9\%$^{*}$} \\
  Plasma-16D & 16 & -92.785$\pm$24.954 & -84.005$\pm$20.141 & \textcolor{ForestGreen}{+9.5\%$^{**}$} & -85.676$\pm$22.113 & \textcolor{ForestGreen}{+7.7\%$^{*}$} & -103.451$\pm$38.509 & \textcolor{red}{-11.5\%$^{*}$} & \textbf{-82.642$\pm$21.800} & \textcolor{ForestGreen}{+10.9\%$^{**}$} \\
  Levy-10D & 10 & -17.193$\pm$11.645 & -13.985$\pm$9.807 & \textcolor{ForestGreen}{+18.7\%$^{*}$} & -13.675$\pm$8.420 & \textcolor{ForestGreen}{+20.5\%$^{**}$} & -20.920$\pm$11.493 & \textcolor{red}{-21.7\%$^{*}$} & \textbf{-13.069$\pm$9.307} & \textcolor{ForestGreen}{+24.0\%$^{**}$} \\
  Levy-20D & 20 & -64.416$\pm$22.656 & -62.420$\pm$29.474 & +3.1\% & -58.497$\pm$23.673 & \textcolor{ForestGreen}{+9.2\%$^{*}$} & -70.013$\pm$22.260 & -8.7\% & \textbf{-58.167$\pm$27.639} & \textcolor{ForestGreen}{+9.7\%$^{*}$} \\
  Ackley-8D & 8 & \textbf{-1.168$\pm$0.756} & -1.350$\pm$0.773 & \textcolor{red}{-15.6\%$^{*}$} & -1.265$\pm$0.786 & -8.3\% & -1.891$\pm$0.776 & \textcolor{red}{-61.9\%$^{**}$} & -1.184$\pm$0.739 & -1.4\% \\
  Hartmann-6D & 6 & \textbf{2.959$\pm$0.341} & 2.818$\pm$0.444 & \textcolor{red}{-4.8\%$^{**}$} & 2.810$\pm$0.513 & \textcolor{red}{-5.0\%$^{**}$} & 2.611$\pm$0.586 & \textcolor{red}{-11.8\%$^{**}$} & 2.934$\pm$0.456 & -0.9\% \\
  CarCab-7D & 7 & -3.888$\pm$0.081 & -3.900$\pm$0.084 & -0.3\% & -3.894$\pm$0.091 & -0.2\% & -3.895$\pm$0.083 & -0.2\% & \textbf{-3.889$\pm$0.088} & -0.0\% \\
  Vehicle-5D & 5 & -5038.884$\pm$0.261 & -5038.994$\pm$0.945 & -0.0\% & -5039.100$\pm$0.993 & -0.0\% & -5039.640$\pm$1.875 & \textcolor{red}{-0.0\%$^{*}$} & \textbf{-5038.873$\pm$0.136} & +0.0\% \\
  RobotPush-14D & 14 & 3.576$\pm$1.799 & 3.639$\pm$1.851 & +1.8\% & 3.736$\pm$1.665 & +4.5\% & 2.691$\pm$1.820 & \textcolor{red}{-24.8\%$^{**}$} & \textbf{3.776$\pm$1.810} & +5.6\% \\
\midrule
  \textit{W / L / N} &  &  & \multicolumn{2}{c}{5 / 2 / 4} & \multicolumn{2}{c}{5 / 1 / 5} & \multicolumn{2}{c}{0 / 7 / 4} & \multicolumn{2}{c}{5 / 0 / 6} \\
\bottomrule
\end{tabular}}
\end{table*}

\Cref{tab:main} compares all methods on the primary initialization (pool\_n6\_k15, 60 seeds). Static KS wins on 5 benchmarks (DTLZ2-16D $+24.5\%$, DTLZ2-8D $+23.2\%$, DTLZ2-20D $+17.8\%$, Plasma $+9.5\%$, Levy-10D $+18.7\%$) but loses on Ackley-8D ($-15.6\%$) and Hartmann-6D ($-4.8\%$). Connected Querying achieves 0W/7L/4N, with $-93.9\%$ on DTLZ2-8D and $-61.9\%$ on Ackley-8D (both $p < 0.01$); see \Cref{app:cpbo} for details.

We ablate the correction design in \Cref{tab:ablation} (\Cref{app:ablation}). Uniform jitter ($\Rmat = cI$): 0W/5L. Uniform diagonal ($\Rmat = \eta^* I$, without $(\Kinv)_{ii}$ weighting): 3W/2L. MetricOnly ($\Rmat$ in the Newton step but $J_0$ unchanged, so MAP is the same): 1W/0L. CovOnly (MAP unchanged, covariance uses $(\mH + \Rmat)^{-1}$): 1W/0L. These results indicate that the MAP correction is the main source of improvement. We also compare against EP-EI and Gibbs-EI~\citep{takeno2023practical}. As discussed in \Cref{sec:related}, changing the inference method does not resolve the rank deficiency of the likelihood Hessian (\Cref{prop:pbo_vulnerable}). Both score worse than Laplace+EUBO (\Cref{app:ep_gibbs}); EP and Gibbs do not have an acquisition function like EUBO; hence, they use EI, designed for scalar observations rather than pairwise comparisons.

\subsection{Main results}
\label{sec:main_results}

Adaptive KS scores \textbf{5W/0L/6N} on the primary initialization: DTLZ2-16D ($+15.1\%$), DTLZ2-20D ($+11.9\%$), Plasma-16D ($+10.9\%$), Levy-10D ($+24.0\%$), Levy-20D ($+9.7\%$). On Ackley-8D and Hartmann-6D, where Static KS loses, Adaptive KS is neutral ($-1.4\%$ and $-0.9\%$, respectively). Fixed20 still loses on Hartmann ($-5.0\%$), confirming that the decisiveness-based activation is important. \Cref{fig:forest} shows effect sizes; convergence curves are given in \Cref{app:convergence}.

\begin{figure}[t]
    \centering
    \includegraphics[width=0.88\textwidth, trim=0 8 0 0, clip]{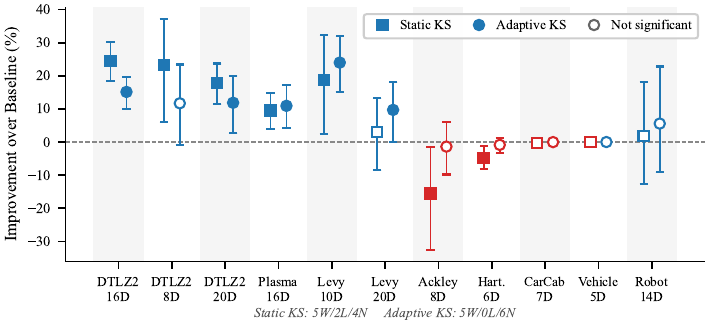}
    \vspace{-4mm}
    \caption{Effect sizes (\% improvement over baseline, 95\% CI, pool\_n6\_k15, 60 seeds). Filled markers indicate $p < 0.05$. Adaptive KS: 5W/0L/6N. Static KS: 5W/2L/4N.}
    \label{fig:forest}
\end{figure}

\paragraph{Robustness across initializations.}
Across all 55 cells (11 benchmarks $\times$ 5 initializations): 18W/2L/35N. Both losses occur on Hartmann-6D under sparse pure-matching initializations and are smaller than $2.5\%$. On pool initializations: 10W/0L. On the plasma problem~\citep{shao2025coactive}: $+10.9\%$ ($p = 0.003$). The threshold $\tau$ is stable: $\tau \in \{0.25, 0.30, 0.35\}$ produces zero losses on 4 representative benchmarks (\Cref{app:tau_sensitivity}).

\section{Discussion and conclusion}
\label{sec:conclusion}

Standard EUBO querying in PBO yields disconnected comparison graphs, causing rank deficiency in the likelihood Hessian and ill-conditioning irrespective of the inference method. KappaSharp addresses this by adding a prior-weighted diagonal correction during MAP estimation, with an adaptive variant that activates based on decisiveness. Across 11 benchmarks, including a plasma medicine task, Adaptive KappaSharp consistently improves performance, with gains up to 
+10.9\%. Remaining limitations include an empirically chosen threshold and evaluation with oracle-simulated preferences. Future work will develop eigenvector-aligned (non-diagonal) corrections, extend analysis to larger batch sizes, and automate threshold selection.

\bibliography{references}
\bibliographystyle{plainnat}

\newpage
\appendix
\section{Proofs}
\label{app:proofs}

\subsection{Proof of \texorpdfstring{\Cref{prop:graph_laplacian}}{Proposition} ($\Hlik$ is a Weighted Graph Laplacian)}

\begin{proof}
For comparison $m$, the probit model (\Cref{eq:probit}) gives $P(x^+_m \succ x^-_m) = \Phi(s_m)$ with $s_m = \mC_m^\top \vf / (\probitscale\sqrt{2})$, so the log-likelihood contribution is $\mathcal{L}_m(\vf) = \log \Phi(s_m)$. The gradient is
\begin{equation}
    \nabla_{\vf} \mathcal{L}_m = \frac{1}{\probitscale\sqrt{2}} \psi(s_m) \mC_m, \quad \text{where } \psi(z) = \frac{\phi(z)}{\Phi(z)}
\end{equation}
is the inverse Mills ratio. Differentiating again:
\begin{equation}
    \nabla^2_{\vf} \mathcal{L}_m = -\frac{1}{2\probitscale^2} \psi(s_m)\bigl(\psi(s_m) + s_m\bigr) \mC_m \mC_m^\top = -\omega_m \mC_m \mC_m^\top.
\end{equation}
\emph{Positivity of $\omega_m$.} For all finite $s_m$, we need $\psi(s_m)(\psi(s_m) + s_m) > 0$. Since $\psi(z) = \phi(z)/\Phi(z) > 0$ for all finite $z$, it suffices to show $\psi(z) + z > 0$. This follows from the classical Mills ratio inequality: $\psi(z) > -z$ for all $z$ (equivalently, $\phi(z)/\Phi(z) > -z$, which holds because $\Phi(z) > 0$ and $\phi(z) > 0$, and for $z \geq 0$ it is immediate; for $z < 0$, the bound $\Phi(z) < \phi(z)/|z|$ gives $\psi(z) > |z| = -z$).

Summing over all $M$ comparisons yields $\Hlik = -\nabla^2 \log p(\mathcal{D} \mid \vf) = \sum_{m=1}^M \omega_m \mC_m \mC_m^\top = \mC^\top \mW \mC$.

\emph{Rank.} Since $\mW$ is positive diagonal (all $\omega_m > 0$), we have $\ker(\mC^\top \mW \mC) = \ker(\mW^{1/2} \mC) = \ker(\mC)$. The kernel of the incidence matrix $\mC$ of a (multi)graph on $N$ vertices with $\Ncc$ connected components is spanned by the $\Ncc$ component-indicator vectors, so $\rank(\mC^\top \mW \mC) = N - \Ncc$.
\end{proof}

\subsection{Proof of \texorpdfstring{\Cref{prop:pbo_vulnerable}}{Likelihood-Null Directions}}

\begin{proof}
Let $\mathbf{Q} = [\mathbf{Q}_0\;\mathbf{Q}_1]$ be orthogonal with $\mathbf{Q}_0$ spanning $\kerfn(\mC)$ and $\mathbf{Q}_1$ spanning $\mathrm{range}(\mC^\top)$. Write $\vf = \mathbf{Q}_0 \xi + \mathbf{Q}_1 \zeta$.

\emph{(a)} Since $\mC\mathbf{Q}_0 = \mathbf{0}$, every comparison score $\mC_m^\top \vf = \mC_m^\top \mathbf{Q}_1 \zeta$ depends only on $\zeta$. Hence $p(\mathcal{D} \mid \xi, \zeta) = p(\mathcal{D} \mid \zeta)$.

\emph{(b)} By Bayes' rule, $p(\xi, \zeta \mid \mathcal{D}) \propto p(\mathcal{D} \mid \zeta)\, p_0(\xi, \zeta) = p_0(\xi \mid \zeta)\, p(\mathcal{D} \mid \zeta)\, p_0(\zeta)$. Normalizing gives $p(\xi, \zeta \mid \mathcal{D}) = p_0(\xi \mid \zeta)\, p(\zeta \mid \mathcal{D})$, so $p(\xi \mid \zeta, \mathcal{D}) = p_0(\xi \mid \zeta)$.

Under the Laplace approximation, write the prior precision in $(\xi,\zeta)$-coordinates as $\mathbf{Q}^\top \Kinv \mathbf{Q} = \bigl[\begin{smallmatrix} \mathbf{A} & \mathbf{B} \\ \mathbf{B}^\top & \mathbf{D}_0 \end{smallmatrix}\bigr]$. The likelihood Hessian contributes only to the $\zeta$-block: $\mathbf{Q}^\top \mH \mathbf{Q} = \bigl[\begin{smallmatrix} \mathbf{A} & \mathbf{B} \\ \mathbf{B}^\top & \mathbf{D}_0 + \mathbf{L} \end{smallmatrix}\bigr]$ with $\mathbf{L} = \mathbf{Q}_1^\top \Hlik \mathbf{Q}_1 \succeq 0$. The MAP satisfies the first-order condition $\mathbf{A}\hat{\xi} + \mathbf{B}\hat{\zeta} = \mathbf{0}$ (since the likelihood does not depend on $\xi$). The Gaussian conditional on $\xi$ given $\zeta$ therefore has mean $\hat{\xi} - \mathbf{A}^{-1}\mathbf{B}(\zeta - \hat{\zeta}) = -\mathbf{A}^{-1}\mathbf{B}\zeta$ and covariance $\mathbf{A}^{-1}$, matching the prior conditional $p_0(\xi \mid \zeta)$.

If $\mathbf{B} = \mathbf{0}$ (block-diagonal prior), then $p_0(\xi \mid \zeta) = p_0(\xi)$, so $p_L(\xi \mid \mathcal{D}) = p_0(\xi)$ and the null-space covariance is unchanged.
\end{proof}

\subsection{Rank Invariance under MAP Shift}

\begin{proposition}[Rank invariance]
\label{prop:rank_invariance_statement}
For every finite $\vf$, $\kerfn(\Hlik(\vf)) = \kerfn(\mC)$ and $\rank(\Hlik(\vf)) = N - \Ncc$, since $\mW(\vf) \succ 0$.
\end{proposition}

\begin{proof}
By \Cref{prop:graph_laplacian}, $\Hlik(\vf) = \mC^\top \mW(\vf)\mC$ with $\mW(\vf) \succ 0$. Then $v^\top \Hlik(\vf) v = \|\mW(\vf)^{1/2}\mC v\|^2$, so $v \in \kerfn(\Hlik(\vf))$ iff $\mC v = 0$. Hence $\kerfn(\Hlik(\vf)) = \kerfn(\mC)$ for all finite $\vf$, and $\rank(\Hlik(\vf)) = N - \Ncc$ by rank-nullity.
\end{proof}

\subsection{Proof of \texorpdfstring{\Cref{prop:filter_properties}}{Proposition} (Correction Properties)}

\begin{proof}
Let $a_i = (\mK^{-1})_{ii} > 0$ and $r_i(\eta) = \eta^2/(\eta + a_i)$ for $\eta > 0$.

\emph{(i) Positivity.} $\eta > 0$ and $\eta + a_i > 0$ immediately give $r_i > 0$.

\emph{(ii) Monotonicity in $\eta$.} $\frac{\partial r_i}{\partial \eta} = \frac{\eta(\eta + 2a_i)}{(\eta + a_i)^2}$. The numerator is positive (both factors positive for $\eta > 0$, $a_i > 0$), and the denominator is positive.

\emph{(iii) Decreasing in $a_i$.} $\frac{\partial r_i}{\partial a_i} = -\frac{\eta^2}{(\eta + a_i)^2} < 0$.

\emph{(iv) Bounded.} $r_i = \eta - \frac{\eta a_i}{\eta + a_i} < \eta$ since $\frac{\eta a_i}{\eta + a_i} > 0$.

\emph{(v)--(vi) Limits.} As $a_i \to 0$: $r_i/\eta = \eta/(\eta + a_i) \to 1$. As $a_i \to \infty$: $r_i/\eta = \eta/(\eta + a_i) \to 0$.

\emph{Variational form.} The objective $q(r) = \frac{a_i}{2}r^2 + \frac{\eta}{2}(r - \eta)^2$ is strictly convex with $q''(r) = a_i + \eta > 0$. Setting $q'(r) = (a_i + \eta)r - \eta^2 = 0$ gives $r^* = \eta^2/(\eta + a_i) = r_i$.
\end{proof}

\subsection{Proof of \texorpdfstring{\Cref{prop:calibration}}{Proposition} ($\kappa$-Calibration is Well-Posed)}

\begin{proof}
\emph{Continuity and PSD.} Each $r_i(\eta)$ is a continuous, non-negative function of $\eta \geq 0$ (with $r_i(0) = 0$), so $\mR(\eta) = \diag(r_1, \ldots, r_N)$ is continuous and PSD.

\emph{Loewner monotonicity.} For $\eta' > \eta \geq 0$: each $r_i(\eta') > r_i(\eta)$ by property (ii), so $\mR(\eta') - \mR(\eta)$ is a positive diagonal matrix, hence $\mR(\eta') \succ \mR(\eta)$ in the Loewner order.

\emph{Limiting behavior.} As $\eta \to \infty$, $r_i(\eta)/\eta \to 1$ for all $i$, so $\mR(\eta) \approx \eta \mI$ for large $\eta$. Since $\mH$ and $\mR(\eta)$ are both symmetric but do not commute in general, we bound $\kappa(\mH + \mR(\eta))$ via Weyl's inequalities. By Weyl's inequality for eigenvalue sums of Hermitian matrices:
\begin{align}
    \lambda_{\max}(\mH + \mR) &\leq \lambda_{\max}(\mH) + \lambda_{\max}(\mR) \leq \lambda_{\max}(\mH) + \eta, \\
    \lambda_{\min}(\mH + \mR) &\geq \lambda_{\min}(\mH) + \lambda_{\min}(\mR) = \lambda_{\min}(\mH) + \min_i r_i(\eta),
\end{align}
where the second line uses that $\mR$ is diagonal with eigenvalues $r_i(\eta)$. Since $r_i(\eta) = \eta^2/(\eta + a_i) = \eta(1 - a_i/(\eta + a_i))$, we have $\min_i r_i(\eta) = \eta(1 - a_{\max}/(\eta + a_{\max}))$ where $a_{\max} = \max_i a_i$. As $\eta \to \infty$, $\min_i r_i(\eta)/\eta \to 1$. Therefore:
\begin{equation}
    \kappa(\mH + \mR(\eta)) \leq \frac{\lambda_{\max}(\mH) + \eta}{\lambda_{\min}(\mH) + \eta(1 - a_{\max}/(\eta + a_{\max}))} \to 1 \quad \text{as } \eta \to \infty.
\end{equation}

\emph{Existence of $\etastar$.} If $\kappa(\mH) = 1$, the target $\kappa(\mH)^{1-\alpha} = 1$ is already met at $\eta^* = 0$. Otherwise, the function $g(\eta) = \kappa(\mH + \mR(\eta))$ is continuous, $g(0) = \kappa(\mH) > 1$, and $g(\eta) \to 1 < \kappa(\mH)^{1-\alpha}$ for $\alpha \in (0,1)$. By the intermediate value theorem, $\{\eta \geq 0 : g(\eta) \leq \kappa(\mH)^{1-\alpha}\}$ is nonempty. As a preimage of a closed set under a continuous function, it is closed; its intersection with $[0, M]$ for sufficiently large $M$ is compact, yielding the minimum $\etastar$.

\begin{remark}[Implementation]
We do not claim monotonicity of $g(\eta)$ in $\eta$; non-commutativity of $\mH$ and $\mR$ prevents this in general. The implementation searches a log-spaced grid of 60 points from $10^{-6}$ to $10^{8}$, identifies the first grid interval where the threshold is crossed, and refines with 30 bisection steps within that interval. The existence result guarantees a valid $\etastar$ exists on $[0, \infty)$; in all experiments, a feasible point was found.
\end{remark}
\end{proof}

\subsection{Proof of \texorpdfstring{\Cref{prop:bias_bound}}{Proposition} (Controlled MAP Shift)}

\begin{proof}
Let $\hat{\vf}_0$ satisfy $\nabla J_0(\hat{\vf}_0) = 0$ (baseline MAP) and $\hat{\vf}_\eta$ satisfy $\nabla J_0(\hat{\vf}_\eta) + \mR\hat{\vf}_\eta = 0$ (augmented MAP). Define $\Delta = \hat{\vf}_\eta - \hat{\vf}_0$.

By the fundamental theorem of calculus:
\begin{equation}
    \nabla J_0(\hat{\vf}_\eta) - \nabla J_0(\hat{\vf}_0) = \left(\int_0^1 \nabla^2 J_0(\hat{\vf}_0 + t\Delta)\,dt\right)\Delta = \bar{\mH}\Delta.
\end{equation}
From the augmented optimality: $\nabla J_0(\hat{\vf}_\eta) = -\mR\hat{\vf}_\eta = -\mR\hat{\vf}_0 - \mR\Delta$. Substituting:
\begin{equation}
    \bar{\mH}\Delta = -\mR\hat{\vf}_0 - \mR\Delta \implies (\bar{\mH} + \mR)\Delta = -\mR\hat{\vf}_0.
\end{equation}
Since $\bar{\mH} \succeq \mK^{-1} \succ 0$ and $\mR \succeq 0$, the matrix $\bar{\mH} + \mR$ is invertible. Setting $\mM = \bar{\mH}^{-1/2}\mR\bar{\mH}^{-1/2}$:
\begin{equation}
    \bar{\mH}^{1/2}\Delta = -(\mI + \mM)^{-1}\mM \bar{\mH}^{1/2}(\hat{\vf}_0).
\end{equation}
Taking norms: $\normH{\Delta} = \|\bar{\mH}^{1/2}\Delta\| \leq \|(\mI + \mM)^{-1}\mM\| \cdot \normH{\hat{\vf}_0}$. Since $\mM \succeq 0$ with largest eigenvalue $\mu_{\max}$, the operator norm $\|(\mI + \mM)^{-1}\mM\| = \mu_{\max}/(1 + \mu_{\max})$.
\end{proof}

\subsection{Proof of \texorpdfstring{\Cref{prop:conservative}}{Proposition} (Full Augmentation Contracts EUBO)}

\begin{proof}
Both posteriors share the MAP $\hat{\vf}_\eta$, so predictive means coincide: $\delta_{ab} = c_{ab}^\top \hat{\vf}_\eta$ for contrast $c_{ab} = e_a - e_b$.

\emph{Variance ordering.} Since $\mR \succeq 0$, $\mH_\eta + \mR \succeq \mH_\eta$, so $(\mH_\eta+\mR)^{-1} \preceq \mH_\eta^{-1}$ (inversion reverses Loewner order on the PD cone). Therefore $s_{\mathrm{aug}}^2(ab) = c_{ab}^\top(\mH_\eta+\mR)^{-1}c_{ab} \leq c_{ab}^\top \mH_\eta^{-1} c_{ab} = s_{\mathrm{MO}}^2(ab)$.

\emph{EUBO monotonicity.} For $q{=}2$ EUBO with $\mathrm{EUBO}(\delta,s) = \delta\Phi(\delta/s) + s\phi(\delta/s)$ (\Cref{eq:eubo}), differentiate at fixed $\delta$: $\partial\,\mathrm{EUBO}/\partial s = \phi(\delta/s) > 0$ for $s > 0$. (The $z^2\phi(z)$ terms from the chain rule cancel exactly.) Hence $s_{\mathrm{MO}} \geq s_{\mathrm{aug}}$ implies $\mathrm{EUBO}_{\mathrm{MO}} \geq \mathrm{EUBO}_{\mathrm{aug}}$.

\emph{Distortion bound.} By the mean value theorem, $\mathrm{EUBO}_{\mathrm{MO}} - \mathrm{EUBO}_{\mathrm{aug}} = \phi(\delta/\tilde{s})(s_{\mathrm{MO}} - s_{\mathrm{aug}})$ for some $\tilde{s} \in [s_{\mathrm{aug}}, s_{\mathrm{MO}}]$. Since $\phi \leq 1/\sqrt{2\pi}$, we get $\mathrm{EUBO}_{\mathrm{MO}} - \mathrm{EUBO}_{\mathrm{aug}} \leq (s_{\mathrm{MO}} - s_{\mathrm{aug}})/\sqrt{2\pi}$.

\emph{Spectral control.} Let $\mM = \mH_\eta^{-1/2}\mR\mH_\eta^{-1/2}$ with eigenvalues $\lambda_i$. Then $s_{\mathrm{MO}}^2/s_{\mathrm{aug}}^2 \leq 1 + \lambda_{\max}(\mM)$, so the distortion is controlled by the spectrum of $\mH_\eta^{-1/2}\mR\mH_\eta^{-1/2}$.

\emph{Extension to fresh query points.}
Let $x_a, x_b \in \mathcal{X}$ be arbitrary query points (possibly $x_a, x_b \notin X$), and define $r_{ab} := k(X, x_a) - k(X, x_b) \in \R^N$ and $\nu_{ab} := k(x_a,x_a) + k(x_b,x_b) - 2k(x_a,x_b) - r_{ab}^\top \mK^{-1} r_{ab} \geq 0$ (nonnegative by the Schur complement of the joint GP prior). Under a Laplace posterior $\mathcal{N}(\hat{\vf}_\eta, \Sigma)$ on the training latents, the GP conditional gives $f(x_a) - f(x_b) \mid \vf \sim \mathcal{N}(r_{ab}^\top \mK^{-1} \vf, \nu_{ab})$, so by total expectation and variance:
\begin{equation}
\delta_{ab} = r_{ab}^\top \mK^{-1} \hat{\vf}_\eta, \quad s_\Sigma^2(ab) = \nu_{ab} + r_{ab}^\top \mK^{-1} \Sigma\, \mK^{-1} r_{ab}.
\end{equation}
For MAP-only, $\Sigma_{\mathrm{MO}} = \mH_\eta^{-1}$; for full augmentation, $\Sigma_{\mathrm{aug}} = (\mH_\eta + \mR)^{-1}$. Since $(\mH_\eta + \mR)^{-1} \preceq \mH_\eta^{-1}$ and congruence preserves Loewner order, $r_{ab}^\top \mK^{-1}(\mH_\eta+\mR)^{-1}\mK^{-1} r_{ab} \leq r_{ab}^\top \mK^{-1}\mH_\eta^{-1}\mK^{-1} r_{ab}$, giving $s_{\mathrm{aug}}^2(ab) \leq s_{\mathrm{MO}}^2(ab)$. The mean $\delta_{ab}$ is identical (shared MAP). EUBO monotonicity then gives $\mathrm{EUBO}_{\mathrm{aug}}(a,b) \leq \mathrm{EUBO}_{\mathrm{MO}}(a,b)$ for all pairs, including fresh query points. When $x_a, x_b \in X$, $r_{ab} = \mK(e_a - e_b)$ and $\nu_{ab} = 0$, recovering the in-sample case.
\end{proof}

\subsection{Proof of \texorpdfstring{\Cref{thm:pseudo_confidence}}{Theorem} (Pseudo-Confidence in Null Directions)}

\begin{proof}
In $(\xi,\zeta)$-coordinates, the MAP-only precision is $\mathbf{Q}^\top \mH_\eta \mathbf{Q} = \bigl[\begin{smallmatrix} \mathbf{A} & \mathbf{B} \\ \mathbf{B}^\top & \mathbf{D}_\eta \end{smallmatrix}\bigr]$ (with $\mathbf{A} = \mathbf{Q}_0^\top \Kinv \mathbf{Q}_0$, since $\mathbf{Q}_0^\top \Hlik \mathbf{Q}_0 = \mathbf{0}$), and the augmented precision adds $\mathbf{Q}^\top \Rmat \mathbf{Q} = \bigl[\begin{smallmatrix} \mathbf{R}_{00} & \mathbf{R}_{01} \\ \mathbf{R}_{01}^\top & \mathbf{R}_{11} \end{smallmatrix}\bigr]$. The Gaussian conditional covariance of $\xi$ given $\zeta$ equals the inverse of the $(\xi,\xi)$-block of the precision. For MAP-only this is $\mathbf{A}^{-1}$; for full augmentation it is $(\mathbf{A} + \mathbf{R}_{00})^{-1}$. Since $\Rmat \succ 0$ and $\mathbf{Q}_0$ has full column rank, $\mathbf{R}_{00} = \mathbf{Q}_0^\top \Rmat \mathbf{Q}_0 \succ 0$, so $(\mathbf{A} + \mathbf{R}_{00})^{-1} \prec \mathbf{A}^{-1}$.
\end{proof}

\subsection{Proof of \texorpdfstring{\Cref{thm:one_step_regret}}{Theorem} (Argmax Stability)}

\begin{proof}
Let $(a^*_{\mathrm{aug}}, b^*_{\mathrm{aug}})$ maximize $\mathrm{EUBO}_{\mathrm{aug}}$. Telescoping gives $\mathrm{EUBO}_{\mathrm{MO}}(a^*_{\mathrm{MO}}, b^*_{\mathrm{MO}}) - \mathrm{EUBO}_{\mathrm{MO}}(a^*_{\mathrm{aug}}, b^*_{\mathrm{aug}}) \leq 2\sup_{a,b} |\mathrm{EUBO}_{\mathrm{MO}}(a,b) - \mathrm{EUBO}_{\mathrm{aug}}(a,b)|$. By the mean value theorem and $\partial\,\mathrm{EUBO}/\partial s = \phi(\delta/s) \leq 1/\sqrt{2\pi}$, each pointwise gap is at most $(s_{\mathrm{MO}} - s_{\mathrm{aug}})/\sqrt{2\pi}$. For each pair $(a,b)$, let $b_{ab}$ denote the contrast vector such that $s_{\mathrm{MO}}^2(a,b) = \nu_{ab} + b_{ab}^\top \mH_\eta^{-1} b_{ab}$ (with $\nu_{ab} \geq 0$ the Schur-complement term from \Cref{prop:conservative}). Then $s_{\mathrm{aug}}^2 \geq s_{\mathrm{MO}}^2/(1+\lambda_1)$, so $s_{\mathrm{MO}} - s_{\mathrm{aug}} \leq s_{\mathrm{MO}}\lambda_1/2$. Combining: the gap is at most $\sup_{a,b} s_{\mathrm{MO}}(a,b)\lambda_1/\sqrt{2\pi}$. The supremum is finite when the kernel $k$ is continuous and $\mathcal{X}$ is compact: by the Cauchy--Schwarz inequality for positive-definite kernels, $|k(x,y)| \leq \sqrt{k(x,x)k(y,y)}$, so $s_{\mathrm{MO}}^2(a,b) \leq 4\sup_x k(x,x)$ for all $(a,b)$. Since $\lambda_1 \leq \|\Rmat\|/\lmin(\mH_\eta) < \eta/\lmin(\mH_\eta)$, the gap is $O(\eta/\lmin(\mH_\eta))$.
\end{proof}

\subsection{Frozen-Utility Monotonicity}

\begin{proposition}[Frozen-utility monotonicity]
\label{prop:weak_mode_mass}
Fix latent utilities $\vf \in \R^N$. Let $\Hlik(\vf) = \sum_{m=1}^M \omega_m(\vf)\,\mC_m\mC_m^\top$ and let $\Hlik'(\vf)$ denote the Hessian after removing comparison $m'$, with all weights evaluated at the same fixed $\vf$. Then $\Hlik'(\vf) = \Hlik(\vf) - \omega_{m'}(\vf)\,\mC_{m'}\mC_{m'}^\top \preceq \Hlik(\vf)$.
This is a frozen-curvature statement: it does not compare Hessians obtained after refitting the MAP on the reduced dataset, since refitting changes the remaining weights.
\end{proposition}

\begin{proof}
$\omega_{m'}(\vf)\,\mC_{m'}\mC_{m'}^\top \succeq 0$ since $\omega_{m'}(\vf) > 0$, so $\Hlik(\vf) - \Hlik'(\vf) \succeq 0$.
\end{proof}

\section{Mechanism Ablation}
\label{app:ablation}

\begin{table*}[t]
\centering
\caption{Mechanism ablation: gain (\%) over baseline B (pool\_n6\_k15, 60 seeds).
Static KS: full prior-gated diagonal correction. Jitter: uniform $R{=}cI$ (no prior gating).
Diagonal: uniform diagonal regularization ($\hat{R}{=}\eta I$).
MetricOnly / CovOnly: partial mechanisms.
$^{*}$: $p{<}0.05$, $^{**}$: $p{<}0.01$ (paired $t$-test vs.\ B).}
\label{tab:ablation}
\small
\begin{tabular}{l c c c c c c}
\toprule
Benchmark & $d$ & Static KS & Jitter & Diagonal & MetricOnly & CovOnly \\
\midrule
  DTLZ2-16D & 16 & \textcolor{ForestGreen}{\textbf{+24.5}$^{**}$} & \textcolor{red}{-22.6$^{**}$} & \textcolor{ForestGreen}{\textbf{+30.7}$^{**}$} & +4.3 & +7.0 \\
  DTLZ2-8D & 8 & \textcolor{ForestGreen}{\textbf{+23.2}$^{*}$} & \textcolor{red}{-47.8$^{**}$} & +1.0 & +13.3 & -14.9 \\
  DTLZ2-20D & 20 & \textcolor{ForestGreen}{\textbf{+17.8}$^{**}$} & \textcolor{red}{-21.0$^{**}$} & \textcolor{ForestGreen}{\textbf{+18.9}$^{**}$} & +4.7 & +2.4 \\
  Plasma-16D & 16 & \textcolor{ForestGreen}{\textbf{+9.5}$^{**}$} & +0.4 & \textcolor{ForestGreen}{\textbf{+8.9}$^{*}$} & +2.3 & -2.9 \\
  Levy-10D & 10 & \textcolor{ForestGreen}{\textbf{+18.7}$^{*}$} & -10.3 & +10.5 & \textcolor{ForestGreen}{\textbf{+18.5}$^{*}$} & \textcolor{ForestGreen}{\textbf{+20.9}$^{*}$} \\
  Ackley-8D & 8 & \textcolor{red}{-15.6$^{*}$} & \textcolor{red}{-45.0$^{**}$} & \textcolor{red}{-22.8$^{**}$} & +2.2 & -15.4 \\
  Hartmann-6D & 6 & \textcolor{red}{-4.8$^{**}$} & \textcolor{red}{-7.0$^{**}$} & \textcolor{red}{-6.7$^{**}$} & +0.2 & -2.1 \\
  Vehicle-5D & 5 & -0.0 & -0.0 & -0.0 & +0.0 & -0.0 \\
\midrule
  \textit{W / L / N} &  & 5/2/1 & 0/5/3 & 3/2/3 & 1/0/7 & 1/0/7 \\
\bottomrule
\end{tabular}
\end{table*}

Uniform jitter adds a fixed constant $c$ to every diagonal entry of the posterior precision ($R{=}cI$, where $c$ is chosen to match the trace of the KS correction); it scores 0W/5L. Diagonal regularization ($\hat{R}{=}\eta I$) uses the $\kappa$-calibrated strength $\eta^*$ but applies it uniformly without the $(\mK^{-1})_{ii}$ weighting; it scores 3W/2L, with $+30.7\%$ on DTLZ2-16D but the same losses as Static KS.

MetricOnly adds $\Rmat$ to the Hessian used in the Newton step for computing the MAP, but evaluates $J_0$ (the original objective without $\Rmat$) at each Newton iterate. Since the objective being minimized is unchanged, the MAP estimate is the same as the baseline; only the Newton search direction is affected. CovOnly keeps the standard MAP estimate $\hat{\vf}_0$ but uses $(\mH + \Rmat)^{-1}$ instead of $\mH^{-1}$ as the Laplace covariance. Both score 1W/0L/7N, indicating that the improvement comes from shifting the MAP estimate (as in the full KappaSharp correction), not from changing the Newton search direction or tightening the covariance alone.

\section{Initialization Robustness}
\label{app:init_robustness}

\begin{figure}[h]
    \centering
    \includegraphics[width=0.95\textwidth]{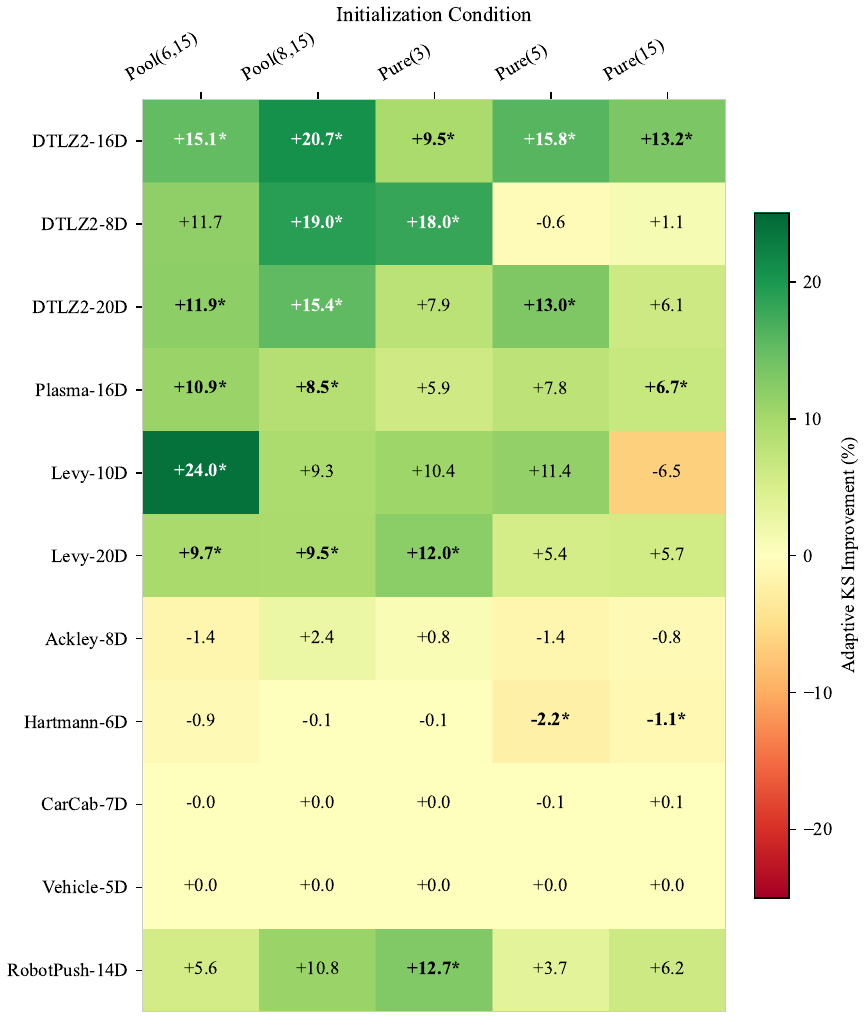}
    \caption{Adaptive KS improvement (\%) over baseline for each of the 55 benchmark--initialization cells. Each cell shows the mean improvement over 60 seeds. Bold with asterisk ($^*$): $p < 0.05$ (paired $t$-test). Green = improvement, red = regression. The two losses (Hartmann-6D under pure\_k5 and pure\_k15) are both below $2.5\%$.}
    \label{fig:heatmap}
\end{figure}

\subsection{Proof of \texorpdfstring{\Cref{prop:info_dim}}{Proposition} (Effective Information Dimension)}

\begin{proof}
Let $\tilde{\mM} = \mK^{1/2}\Hlik\mK^{1/2} \succeq 0$ with eigenvalues $\mu_1 \geq \cdots \geq \mu_N \geq 0$. Since $\Hlik$ has rank $N{-}\Ncc$ (\Cref{prop:graph_laplacian}), $\tilde{\mM}$ has exactly $\Ncc$ zero eigenvalues.

\emph{(a)} $d_{\mathrm{eff}} = \sum_{i=1}^N \mu_i/(1+\mu_i)$. Each term satisfies $0 \leq \mu_i/(1+\mu_i) < 1$, with equality to $0$ iff $\mu_i = 0$. Since exactly $\Ncc$ eigenvalues are zero, at most $N{-}\Ncc$ terms are positive, giving $d_{\mathrm{eff}} \leq N{-}\Ncc$. Equality requires all nonzero $\mu_i \to \infty$.

\emph{(b)} Under the fresh-pair assumption with $N_{\mathrm{cc},0}$ initial connected components: each of the $T$ queries adds two unseen points and one new component, so $\Ncc = N_{\mathrm{cc},0} + T$ and $N = 2T + N_0$. Therefore $\rank(\Hlik) = N - \Ncc = N_0 + T - N_{\mathrm{cc},0}$ and $d_{\mathrm{eff}} \leq N_0 + T - N_{\mathrm{cc},0}$. Hence $d_{\mathrm{eff}}/N \leq (N_0 + T - N_{\mathrm{cc},0})/(2T + N_0) \to 1/2$ as $T \to \infty$ (with $N_0, N_{\mathrm{cc},0}$ fixed).

\emph{(c)} Connected ($\Ncc{=}1$): $d_{\mathrm{eff}} \leq N{-}1$. The rank ceiling is lifted to $N{-}1$, though the realized value depends on the magnitude of comparison weights (a connected graph with vanishing edge weights can still have $d_{\mathrm{eff}} \ll N{-}1$).

\emph{(d)} Let $\mB = \mK^{1/2}\Rmat\mK^{1/2} \succeq \beta\mI$ with $\beta > 0$. Define $g(\mathbf{X}) = \tr\bigl(\mathbf{X}(\mI+\mathbf{X})^{-1}\bigr) = N - \tr\bigl((\mI+\mathbf{X})^{-1}\bigr)$. Since $g$ is monotone increasing under the Loewner order (adding PSD mass decreases $(\mI+\mathbf{X})^{-1}$), and $\tilde{\mM} + \mB \succeq \tilde{\mM} + \beta\mI$, we have $d_{\mathrm{eff}}^{\mathrm{aug}} = g(\tilde{\mM}+\mB) \geq g(\tilde{\mM} + \beta\mI)$. Since $\beta\mI$ commutes with $\tilde{\mM}$:
\begin{equation}
g(\tilde{\mM} + \beta\mI) = \sum_{i=1}^N \frac{\mu_i + \beta}{1 + \mu_i + \beta} = d_{\mathrm{eff}} + \sum_{i=1}^N \frac{\beta}{(1+\mu_i)(1+\mu_i+\beta)}.
\end{equation}
Each term in the sum is nonnegative, and the $\Ncc$ terms with $\mu_i = 0$ each contribute $\beta/(1{+}\beta)$, giving the bound $d_{\mathrm{eff}}^{\mathrm{aug}} \geq d_{\mathrm{eff}} + \Ncc\beta/(1{+}\beta)$.
\end{proof}

\section{Threshold Interpretation}
\label{app:threshold_interp}

\begin{proposition}[Per-step decisiveness threshold]
\label{prop:margin_safe}
Since $p_t = \Phi(|\cdot|) \in [1/2, 1)$ by construction, the mapping $p \mapsto s_{\mathrm{exp}} = 1 - \mathcal{H}(p)/\log 2$ is strictly increasing on $[1/2, 1)$. Hence $s_{\mathrm{exp},t} \geq \tau$ implies $p_t \geq p(\tau)$. At $\tau = 0.30$: $p_t \geq 0.81$.
\end{proposition}

\begin{proof}
$\mathcal{H}(p)/\log 2$ is maximized at $p = 0.5$ and strictly decreasing on $(0.5, 1)$. Solving $1 - \mathcal{H}(p)/\log 2 = 0.30$ yields $p \approx 0.8107$.
\end{proof}

On multi-objective benchmarks where the disconnected-graph problem is severe, the model quickly becomes confident about queried comparisons and $\bar{s}_t$ rises above $\tau$, activating the correction. On smooth single-objective basins (e.g., Ackley), the model remains uncertain during exploration and $\bar{s}_t$ stays low, keeping the correction off.

\section{Convergence Curves}
\label{app:convergence}

\begin{figure}[h]
    \centering
    \includegraphics[width=0.95\textwidth]{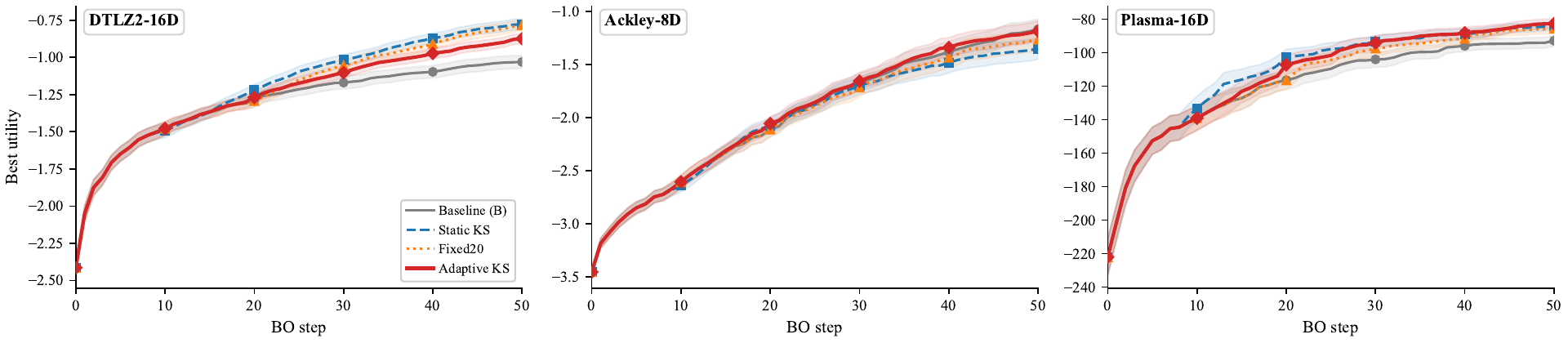}
    \caption{Convergence (mean $\pm$ s.e., 60 seeds). Left: DTLZ2-16D. Center: Ackley-8D. Right: Plasma-16D ($+10.9\%$).}
    \label{fig:convergence}
\end{figure}

\section{Activation Behavior}
\label{app:gate_behavior}

\begin{figure}[h]
    \centering
    \includegraphics[width=0.95\textwidth]{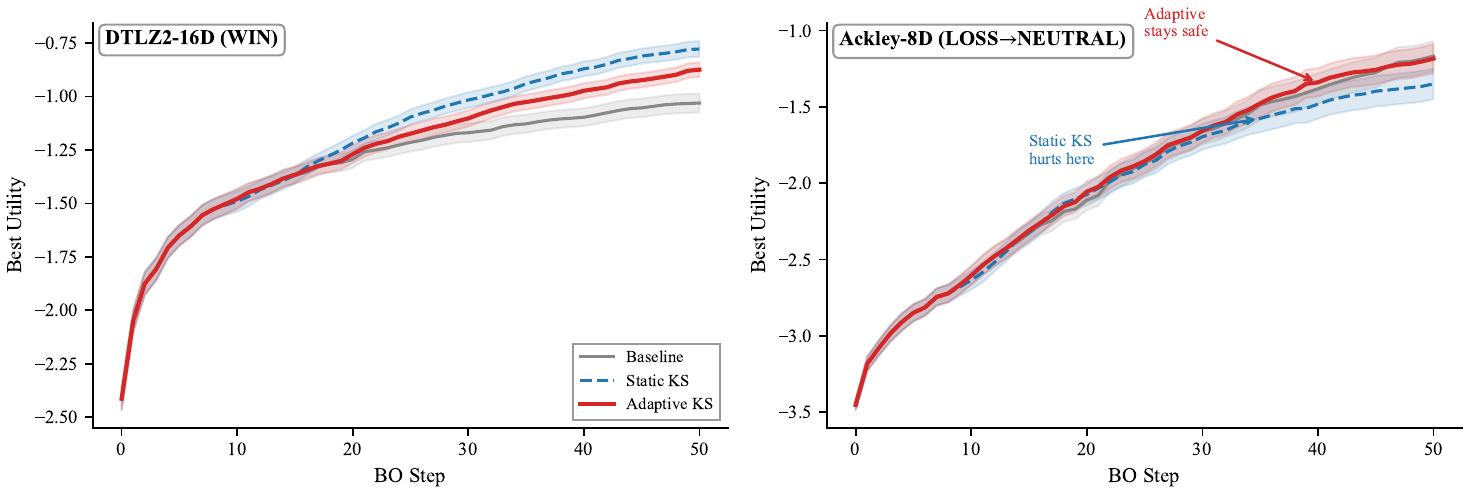}
    \caption{Activation behavior on WIN (DTLZ2-16D) vs.\ LOSS$\to$NEUTRAL (Ackley-8D) benchmarks. On DTLZ2-16D, the activation rule enables the correction frequently (median $\approx 40\%$ of steps); on Ackley-8D, low model decisiveness keeps the correction largely off (median $\approx 12\%$), preventing the MAP shift that causes Static KS's loss.}
    \label{fig:gate}
\end{figure}

\section{Additional Experimental Details}
\label{app:details}

\subsection{Benchmark Descriptions}

\paragraph{DTLZ2 (8D, 16D, 20D).}
Standard multi-objective test function~\citep{gonzalez2017preferential} with $d$-dimensional input mapped to a 4-dimensional outcome space, then to scalar utility via a piecewise-linear aggregation. The Pareto front is a unit hypersphere in outcome space.

\paragraph{Plasma-16D.}
Real industrial application: calibrating a 16-parameter model predictive controller for atmospheric plasma surface treatment~\citep{shao2025coactive}. Each function evaluation solves an MPC optimization with physics simulation. The utility combines temperature uniformity and treatment speed, assessed by human inspection of temperature distributions.

\paragraph{Ackley-8D.}
Standard global optimization benchmark on $[-1,1]^8$ with many shallow local optima surrounding a single global minimum. The smooth basin structure provides strong local curvature, making the Laplace approximation already well-conditioned.

\paragraph{Hartmann-6D.}
Standard 6-dimensional test function with 6 local minima. Like Ackley, the smooth landscape provides sufficient local curvature for the baseline.

\paragraph{Levy-10D, Levy-20D.}
Multi-modal test functions with increasingly complex landscape at higher dimensions. The 20D variant tests scalability.

\paragraph{CarCab-7D, Vehicle-5D, RobotPush-14D.}
CarCab: 7-dimensional vehicle cabin design. Vehicle: 5-dimensional crash safety optimization. RobotPush: 14-dimensional robotic pushing task simulated with Box2D physics, featuring contact discontinuities.

\subsection{Hyperparameters}

All experiments use BoTorch 0.6.0, GPyTorch 1.6.0, PyTorch 1.11.0. PairwiseGP uses default RBF kernel with automatic relevance determination (ARD). EUBO is optimized with 4 restarts and 32 raw samples. The precision correction activates at step~8 with $\alpha = 0.1$. The $\etastar$ solve uses a grid search on 60 log-spaced points from $10^{-6}$ to $10^8$.

Init tables are generated with hash-based independent RNG per (benchmark, seed, stream) using \texttt{gen\_init\_tables\_v3.py}. Pool initializations use nested point sets ($n_6 \subset n_8$) with independently shuffled comparison orderings.

\paragraph{Activation step.}
The correction is eligible from step~8 by default. In a preliminary study using Static KS on DTLZ2-16D (30 seeds, separate from the main 60-seed runs), we tested eligibility at steps 5, 8, 10, and 12; final improvement ranged from $+19\%$ to $+24\%$, indicating low sensitivity to the exact activation step. Earlier eligibility (step~5) occasionally produced unstable $\etastar$ estimates; later eligibility (step~12) delayed the benefit without improving it. All main-table results use step~8 with 60 seeds.

\paragraph{Compute.}
Experiments were run on a university HPC cluster (AMD EPYC 7543 and Intel Xeon Platinum CPUs, 64-core nodes). Total compute: approximately 5{,}000 CPU-hours across $16{,}500$ BO runs in the main grid plus ablations and sensitivity studies. All 60-seed conditions use paired same-node execution with common random numbers (CRN).

\section{Threshold Sensitivity}
\label{app:tau_sensitivity}

\begin{table}[h]
\centering
\caption{Threshold sensitivity: Adaptive KS gain (\%) over baseline for $\tau \in \{0.25, 0.30, 0.35\}$ on 4 representative benchmarks (pool\_n6\_k15, 60 seeds). All three thresholds produce zero losses. $^{*}$: $p{<}0.05$, $^{**}$: $p{<}0.01$ (paired $t$-test).}
\label{tab:tau}
\small
\begin{tabular}{l ccc}
\toprule
Benchmark & $\tau = 0.25$ & $\tau = 0.30$ & $\tau = 0.35$ \\
\midrule
DTLZ2-16D & +20.2$^{**}$ & +15.1$^{**}$ & +20.1$^{**}$ \\
Levy-10D & +25.0$^{**}$ & +24.0$^{**}$ & +22.4$^{**}$ \\
Ackley-8D & -2.3 & -1.4 & +4.3 \\
Plasma-16D & +9.9$^{**}$ & +10.9$^{**}$ & +7.0$^{*}$ \\
\midrule
\textit{Losses} & 0 & 0 & 0 \\
\bottomrule
\end{tabular}
\end{table}

The threshold $\tau = 0.30$ has a per-step interpretation: when $s_{\mathrm{exp},t} \geq 0.30$, the model's confidence in its own prediction satisfies $p_t \geq 0.81$ (\Cref{prop:margin_safe}). The implementation triggers on the EMA $\bar{s}_t$ rather than the raw score; since $s_{\mathrm{exp}}$ is nonlinear in $p$, $\bar{s}_t \geq \tau$ does not imply an EMA-weighted average $p \geq 0.81$, but it does indicate that recent steps have exhibited high decisiveness on average. Lower $\tau$ (more aggressive) increases activation rate and gains on WIN benchmarks. All three thresholds produce zero statistically significant losses.

\section{Calibration Parameter $\alpha$ Sensitivity}
\label{app:alpha_sensitivity}

\begin{table}[h]
\centering
\caption{Static KS gain (\%) over baseline for $\alpha \in \{0.05, 0.1, 0.2\}$ on 4 benchmarks (pool\_n6\_k15). Larger $\alpha$ targets a lower condition number, producing stronger augmentation. $\alpha = 0.1$ is a reasonable operating point; the tradeoff is not monotone (e.g., Ackley at $\alpha{=}0.05$ is worse than at $\alpha{=}0.1$). $^{*}$: $p{<}0.05$, $^{**}$: $p{<}0.01$.}
\label{tab:alpha}
\small
\begin{tabular}{l ccc}
\toprule
Benchmark & $\alpha = 0.05$ & $\alpha = 0.1$ & $\alpha = 0.2$ \\
\midrule
DTLZ2-16D & +18.8$^{**}$ & +24.5$^{**}$ & +31.8$^{**}$ \\
Plasma-16D & +6.7 & +9.5$^{**}$ & +17.4$^{**}$ \\
Ackley-8D & $-$31.4$^{*}$ & $-$15.6$^{*}$ & $-$51.0$^{**}$ \\
Hartmann-6D & $-$1.0 & $-$4.8$^{**}$ & $-$11.8$^{**}$ \\
\bottomrule
\end{tabular}
\end{table}

\section{Connected Querying Full Results}
\label{app:cpbo}

\begin{table}[h]
\centering
\caption{Connected querying~\citep{xu2024principled} detailed results (pool\_n6\_k15, 60 seeds). One candidate per query is fixed to the previous query point ($x'_t = x_{t-1}$); EUBO is retained as the acquisition function for fair comparison. $^{*}$: $p < 0.05$, $^{**}$: $p < 0.01$.}
\label{tab:cpbo}
\small
\begin{tabular}{l c cc}
\toprule
Benchmark & $d$ & Conn.\ Query Gain (\%) & Verdict \\
\midrule
DTLZ2-16D & 16 & $+1.5$ & NEUTRAL \\
DTLZ2-8D & 8 & \textcolor{red}{$-93.9^{**}$} & LOSS \\
DTLZ2-20D & 20 & $+6.2$ & NEUTRAL \\
Plasma-16D & 16 & \textcolor{red}{$-11.5^{*}$} & LOSS \\
Levy-10D & 10 & \textcolor{red}{$-21.7^{*}$} & LOSS \\
Levy-20D & 20 & $-8.7$ & NEUTRAL \\
Ackley-8D & 8 & \textcolor{red}{$-61.9^{**}$} & LOSS \\
Hartmann-6D & 6 & \textcolor{red}{$-11.8^{**}$} & LOSS \\
CarCab-7D & 7 & $-0.2$ & NEUTRAL \\
Vehicle-5D & 5 & \textcolor{red}{$-0.0^{*}$} & LOSS \\
RobotPush-14D & 14 & \textcolor{red}{$-24.8^{**}$} & LOSS \\
\midrule
\textit{W / L / N} & & \multicolumn{2}{c}{0 / 7 / 4} \\
\bottomrule
\end{tabular}
\end{table}

Connected querying enforces a connected comparison graph by fixing one candidate per query to the previous query point ($x'_t = x_{t-1}$). Note that the original algorithm of \citet{xu2024principled} uses a different acquisition function based on optimistic likelihood-ratio confidence sets; here we adopt only their querying strategy and retain EUBO as the acquisition function so that the comparison isolates the effect of graph connectivity. Fixing one candidate to the previous point limits exploration, resulting in losses on 7 of 11 benchmarks, including $-93.9\%$ on DTLZ2-8D and $-61.9\%$ on Ackley-8D.

\section{Non-Laplace Inference Baselines}
\label{app:ep_gibbs}

EP-EI and Gibbs-EI results on 4 representative benchmarks (pool\_n6\_k15, 60 seeds):

\begin{table}[h]
\centering
\small
\begin{tabular}{l cc cc}
\toprule
Benchmark & EP-EI Gain & Verdict & Gibbs-EI Gain & Verdict \\
\midrule
DTLZ2-16D & $-75.2\%^{**}$ & LOSS & $+9.6\%$ & NEUTRAL \\
Ackley-8D & $-130.2\%^{**}$ & LOSS & $-52.5\%^{**}$ & LOSS \\
Plasma-16D & $-73.1\%^{**}$ & LOSS & $-40.3\%^{**}$ & LOSS \\
Levy-10D & $-82.9\%^{**}$ & LOSS & $-22.0\%^{*}$ & LOSS \\
\bottomrule
\end{tabular}
\caption{Non-Laplace inference baselines (diagnostic). Both use EI acquisition, not EUBO, which contributes to the poor performance. $^{*}$: $p<0.05$, $^{**}$: $p<0.01$.}
\end{table}

\section{Wall-Clock Overhead}
\label{app:wallclock}

\begin{table}[!h]
\centering
\caption{Wall-clock time (seconds) per 50-step BO run. KappaSharp adds negligible overhead
and is often \emph{faster} due to improved Hessian conditioning.}
\label{tab:wallclock}
\small
\begin{tabular}{l ccc}
\toprule
Benchmark & Baseline (s) & KappaSharp (s) & Overhead \\
\midrule
  DTLZ2-16D & 120$\pm$42 & 148$\pm$43 & +23\% \\
  DTLZ2-8D & 126$\pm$46 & 129$\pm$43 & +2\% \\
  Plasma-16D & 272$\pm$81 & 271$\pm$83 & \textcolor{green!60!black}{-0\%} \\
  Levy-10D & 156$\pm$47 & 146$\pm$38 & \textcolor{green!60!black}{-7\%} \\
  Ackley-8D & 126$\pm$44 & 128$\pm$43 & +2\% \\
  Hartmann-6D & 118$\pm$50 & 102$\pm$36 & \textcolor{green!60!black}{-13\%} \\
  Vehicle-5D & 48$\pm$10 & 54$\pm$9 & +11\% \\
  RobotPush-14D & 831$\pm$84 & 817$\pm$80 & \textcolor{green!60!black}{-2\%} \\
\bottomrule
\end{tabular}
\end{table}

The correction adds per-step overhead when active. On 4 of 8 benchmarks, the method is actually \emph{faster} than baseline ($-2\%$ to $-13\%$ wall-clock) because the better-conditioned Hessian reduces Newton iterations; on others, the extra MAP solve adds up to $+23\%$ wall-clock time.

\section{Full Per-Initialization Numeric Results}
\label{app:init_results}

The heatmap in \Cref{fig:heatmap} summarizes Adaptive KS's improvement across all 55 benchmark--initialization conditions. Numeric values for each cell are available in the figure. Key observations:

\begin{itemize}
    \item \textbf{Pool initializations} (n6\_k15, n8\_k15): 10W/0L/12N across all benchmarks. The nested pool structure ($n_6 \subset n_8$) provides a dense initial comparison graph that combines well with the activation rule.
    \item \textbf{Pure matching initializations} (k3, k5, k15): 8W/2L/23N. Both losses occur on Hartmann-6D (pure\_k5: $-2.2\%$, $p < 0.05$; pure\_k15: $-1.1\%$, $p < 0.05$). These are the smallest significant effects in the entire study, and both occur on a benchmark where the baseline is already well-conditioned.
    \item \textbf{High-dimensional benchmarks} (DTLZ2-16D, DTLZ2-20D, Levy-20D, Plasma-16D): positive across all initializations, which is consistent with the disconnected-graph problem being the main bottleneck in those settings.
    \item \textbf{Low-dimensional benchmarks} (Vehicle-5D, Hartmann-6D, CarCab-7D): mostly neutral. At low dimensionality, the comparison graph provides sufficient curvature relative to the prior, and the correction has little to add.
\end{itemize}

\end{document}